\documentclass[preprint,3p,times]{elsarticle}

\makeatletter
\def\ps@pprintTitle{%
 \let\@oddhead\@empty
 \let\@evenhead\@empty
 \def\@oddfoot{\reset@font\hfil\thepage\hfil}%
 \let\@evenfoot\@oddfoot
 \def\@oddfoot{\footnotesize\itshape [Preprint submitted to Elsevier]\hfill}%
}
\makeatother

\usepackage{amssymb}
\usepackage{amsmath}
\usepackage{booktabs}
\usepackage{algorithm}
\usepackage{rotating}
\usepackage{graphicx}
\usepackage{multirow}
\usepackage{multicol}
\usepackage{algpseudocode}
\usepackage{tikz}
\usepackage{adjustbox}
\usepackage{paralist}
\usepackage{enumitem}
\usepackage{array}
\usepackage{lineno}
\usepackage{tabularx}
\usepackage{url}

\newtheorem{theorem}{Theorem}

\newtheorem{corollary}{Corollary}
\newtheorem{proof}{Proof}

\journal{Elsevier}

\begin{document}

\begin{frontmatter}

\title{Separable Neural Architectures as Physical World Models: from Mathematical Theory to Applications}

\author[label1]{Reza T. Batley}
\ead{rezabatley@vt.edu}

\author[label1]{Andrew Kichline}

\author[label1]{Sourav Saha\corref{cor1}} \ead{souravsaha@vt.edu}
\cortext[cor1]{Corresponding author}

\affiliation[label1]{organization={Kevin T. Crofton Department of Aerospace and Ocean Engineering, Virginia Polytechnic Institute and State University},
            addressline={1600 Innovation Drive}, 
            city={Blacksburg},
            postcode={24060}, 
            state={VA},
            country={United States}}

\begin{abstract}
This work introduces the Separable Neural Architecture (SNA), a function representational class combining neural approximation with tensor decomposition. The SNA decouples localized coordinate functions -- atoms -- from global interactions, governed by a sparse and low-rank interaction object. This architecture possesses a compact and smooth inductive bias, well-suited in the solution of PDEs. Indeed, when viewed as a Galerkin trial space under the \emph{variational} SNA (VSNA) framework, the resulting formulation satisfies variational guarantees under Lax-Milgram: well-posedness, quasi-optimality, convergence and stability. In high-dimensional spatiotemporal--parametric PDEs, the VSNA scales algebraically with dimensionality, as opposed to the exponential explosion in complexity of grid-based methods. Exploiting an entirely factorized, tensor-native alternating least squares (ALS) optimization framework brings the cost down to linear in dimension. The architecture is evaluated across the three PDE classes, elliptic Poisson, hyperbolic inviscid- and viscous Burgers, and a parabolic six-dimensional parametric advection-diffusion system. The VSNA exhibits agreement with predicted algebraic and spectral scaling rates. The SNA's capability as a solve once, query anywhere physical world model is demonstrated in two case studies in engineering science. In the study of a seven-dimensional parametric model of variable beam shape laser powder bed fusion manufacturing process, the VSNA encodes non-separable laser geometries and carries out a 1,000,000-query Monte Carlo optimization sweep under uncertainty in 102s on commodity laptop CPU hardware. This is a $150{,}000\times$ speedup over a traditional full-grid finite element multi-query baseline hosted on an NVIDIA A100 GPU. On sparse experimental directed energy deposition data, the SNA maps thermal histories to Inconel 718 tensile properties with five orders of magnitude fewer parameters than the prior art, whilst enabling generative inverse-mode reconstructions in under 100ms, again on commodity hardware. These results demonstrate that the SNA serves a single, compact mathematical substrate for continuous parameter manifolds to enable real-time inversion, optimization loops and rapid uncertainty propagation.
\end{abstract}



\begin{keyword}
Separable neural architectures \sep Physical world models \sep Variational methods \sep Parametric PDEs \sep Manufacturing \sep Scientific Machine Learning
\end{keyword}

\end{frontmatter}




\section{Introduction}

Predicting the evolution of systems -- under varying loads, geometries, material properties or operating conditions -- is one of the central challenges in engineering science. Classical discretization methods are mathematically rigorous, well-posed and backed by decades of convergence theory. For single queries at fixed parameters in low dimensions ($\leq 3$) -- the spatial dimensions of everyday engineering problems -- they are the gold standard. Their cost is also unambiguous: degrees of freedom scale exponentially with dimension, with solvers typically carrying a further growth factor, a dominant computational burden. Parametric and multi-query problems -- optimization, uncertainty quantification, inverse modeling -- reintroduce the full solver cost at every new parameter evaluation, rendering grid-based approaches untenable even when $d$ is modest \cite{bellman1961}. Physics-informed neural approximators \cite{raissi2019,karniadakis2021} circumvent the grid but carry no structural prior reflecting the physics of the problem, and provide no analogue of the convergence guarantees that make classical methods trustworthy.

The engineering ideal is a single, compact surrogate that behaves as a \emph{physical world model} \cite{ha2018world}. Distinct from agentic AI frameworks -- which focus on sequential decision-making processes -- a physical world model acts as a continuous map of a system's complete physical manifold. It embeds the underlying conservation laws directly into its structural representation, enabling it to internalize the spacetime and parametric landscape. It must be queried instantly at any point in this space; invertible for optimal design; and amenable to uncertainty propagation without re-solving. The representation this demands must be accurate, theoretically certified and parameter-efficient, defined over the entire solution manifold, not just at isolated points.

Many physical fields are, exactly or approximately, separable in their natural variables. Where this structure is present and exploited, representations become compact and interpretable. Tensor decomposition methods \cite{kolda2009,grasedyck2013} have made separability precise for discrete arrays, yielding exponential reductions in storage and enabling computation in regimes intractable for dense formats. A unified neural framework combining explicit interaction control, universal approximation and variational theory has remained largely absent from the literature.

\begin{figure}[!h]
    \centering
    \includegraphics[width=\linewidth]{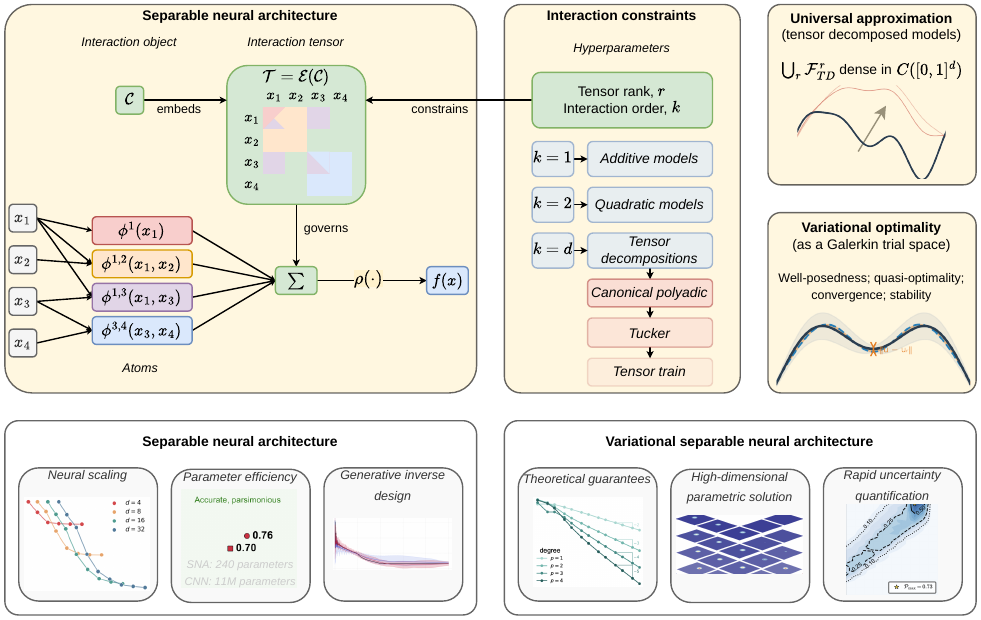}
    \caption{\textbf{Architectural map, theoretical foundations and the dual computational frameworks of the Separable Neural Architecture.} Inputs $x$ is projected into localized coordinate subsets $S$ selected by an interaction tensor embedding $\mathcal{T}=\mathcal{E}(\mathcal{C})$, processed through parallel learnable function blocks -- atoms $\phi^{(S)}$ -- summed and nonlinearity $\rho$ applied to form the approximant $f$. The interaction tensor is constrained by maximum interactivity $k$ and rank $r$.}
    \label{fig:sna}
\end{figure}

This work introduces the \emph{Separable Neural Architecture} (SNA) to fill this gap. Structural interactions within the model are governed by an \emph{interaction object} $\mathcal{C} = \{c_S\}$ admitting a sparse tensor embedding via $\mathcal{E}$, parameterized by interaction order $k$ and rank $r$. This object acts over a collection of learnable functions: \emph{atoms} $\phi^{(S)}$, each depending on a subset $S \subseteq \{1,\ldots,d\}$ of input variables with $|S| \leq k$. This explicit inductive bias restricts parameter allocation to active coordinate interactions. Consequently, the architecture subsumes generalized additive models at $k=1$ \cite{hastie1990}, quadratic interaction models at $k=2$, and CP-decomposed models at $k=d$. The resulting structure is parameter-efficient, universally accurate over continuous domains and possesses tractable, even exact, derivatives. Because it maps over an entire solution manifold, admits rapid inversion and enables real-time uncertainty quantification, the SNA provides the mathematical substrate of a physical world model.

The SNA is extended to the \emph{Variational Separable Neural Architecture} (VSNA), which employs the SNA as a Galerkin trial space over the coupled spatiotemporal--parametric domain $\mathcal{X} = \Omega \times [0,T] \times \mathcal{P}$, with dimension $d = n+1+m$. Well-posedness, quasi-optimality, stability, and convergence of the VSNA Galerkin problem are established via Lax--Milgram Theorem. Universality of the CP-class SNA is proved via Stone--Weierstrass Theorem. Training exploits a tensor-native Alternating Least Squares (ALS) algorithm whose complexity is $O(I_{\max} \cdot d \cdot r^2 \cdot C)$, replacing the $O(N^{3d})$ cost of full-grid discretization.

Figure \ref{fig:sna} provides an overview of the architecture, its theoretical foundations and results developed in this paper. Building on this, the principal contributions of the work are as follows:
\begin{itemize}
    \item The SNA is formally defined as a rank-$r$, order-$k$ interaction architecture parameterized by learnable \emph{atoms} $\varphi$ and a sparse \emph{interaction object} $\mathcal{C}$. Universality over continuous functions on $[0,1]^d$ is proved via the Stone--Weierstrass theorem for the CP-class ($k=d$).
    \item The VSNA is introduced as a Galerkin trial space over coupled spatiotemporal--parametric domains; well-posedness, quasi-optimality, stability and convergence of the resulting variational problem are established under the Lax-Milgram theorem.
    \item A tensor-native ALS training algorithm is derived -- reducing algorithmic complexity from $O(C^{3d})$ to $O(I_\mathrm{max}\cdot d\cdot r^3\cdot C^3)$ for direct-, or $O(I_{\max}\cdot d\cdot r^2\cdot C)$ for iterative solvers.
    \item The SNA and VSNA are validated across an extensive suite of benchmarks. The SNA: a 20-dimensional Sobol-G regression problem, image classification latent scaling sweeps on ImageNet and MNIST, an Inconel 718 directed energy deposition (DED) process-structure predictive and inverse modeling problem; the VSNA: an elliptic Poisson, hyperbolic viscous and inviscid Burgers shock layer, parabolic six-dimensional parametric advection-diffusion problem, and a seven-dimensional parametric shaped-beam laser powder bed fusion (LPBF) process optimization window under uncertainty.
\end{itemize}
The remainder of the paper is structured as follows. Section \ref{sec:background} surveys the classical and neural approximation prior art through which the SNA's contributions should be read. Section \ref{sec:methods} defines the SNA and VSNA, establishes their theoretical properties and introduces the practical implementations used in this work. Section \ref{sec:results} presents numerical experiments across elliptic, hyperbolic and parabolic PDEs, and scaling behavior in supervised machine learning. Section \ref{sec:applications} demonstrates this model on a seven-dimensional problem in laser powder bed fusion for process optimization under uncertainty, and on process--structure forward and inverse modeling in directed energy deposition.

\section{Background}
\label{sec:background}

\subsection{Classical approximation of high-dimensional fields}

The finite element method \cite{strang1973} and spectral discretizations \cite{canuto2007} serve as the methodological core of computational science. Both provide rigorous convergence theories, precise error estimates and well-understood conditioning properties of the resulting linear systems. Across elliptic, parabolic and hyperbolic problems in spatial domains $\Omega\subset\mathbb{R}^d$ with $d\leq 3$, these methods are near-optimal in appropriate Sobolev norms. The difficulty arises as $d$ grows. A uniform grid with $C$ points per coordinate axis requires $O(C^d)$ degrees of freedom (DoFs). Direct solution of the resulting linear system carries cost $O(C^{3d})$. At $d=6$ and $C=100$, the DoF count alone is $10^{12}$. The solver cost is astronomical at $O(10^{36})$. This is the curse of dimensionality \cite{bellman1961}, and it renders grid-based discretization infeasible for any problem in which the relevant dimension exceeds a handful.

This is most pertinent in parametric and multi-query settings. In design optimization, uncertainty quantification and Bayesian inference, the solution must be evaluated over a parameter manifold $\mathcal{P}\subset \mathbb{R}^m$. Each evaluation of a classical solver is a fresh $O(C^{3d})$ computation, and the required number of evaluations typically grows exponentially in $m$. Even when single-query cost is moderate, the aggregate multi-query budget quickly becomes prohibitive. Reduced-order modeling (ROM) was developed to address this. Proper orthogonal decomposition (POD) \cite{sirovich1987} identifies the dominant modes of a solution manifold from a set of high-fidelity snapshots. Reduced basis methods \cite{rozza2008,benner2015} construct certified low-dimensional approximation spaces by greedy selection of representative solutions. Both exploit the empirical observation that PDE solution manifolds -- even with high ambient dimensionality -- often possess low intrinsic dimension. The cost of online evaluation thus reduces to $O(r^3)$ for a reduced space smaller than the full-order model dimension.

These strategies are fundamentally post hoc: they require the offline construction of a full-order model and its repeated solution at chosen parameter values. The parametric sampling itself remains expensive when $m$ is large, and greedy snapshot selection does not, in general, guarantee coverage of the full parameter manifold. The proper generalized decomposition (PGD) \cite{ammar2006pgd,Chinesta2011} employs a different approach, constructing a separated representation of the PDE solution directly within the variational solve, without ever constructing a full-order system. PGD proceeds via greedy single-rank enrichment in an alternating least-squares fashion. At each step, a single additional term $u^{(r)}(x)=\prod_i v_i^{(r)}(x_i)$ is appended to the current approximant. Factors $v_i^{(r)}$ are found by solving a sequence of univariate problems. PGD is the closest classical antecedent to the VSNA: it seeks a separated representation, exploits alternating optimization and operates directly within a variational framework. However, it remains limited: the separated form is fixed a priori as a rank-one-per-enrichment CP product with no learnable atom architecture. It requires an explicit variational formulation and cannot be trained from data alone. It provides no mechanism for controlling variable interaction order. Its reliance on a greedy training strategy prevents communication between rank components, often requiring a higher rank to reach a given accuracy than global training would.

\subsection{Tensor decompositions and interaction-structured representations}

The three canonical tensor formats provide the foundation for separable representations of discrete arrays. The canonical polyadic (CP) decomposition \cite{hitchcock1927,kolda2009} writes a $d$-way tensor  $\mathcal{T}\in\mathbb{R}^{I_1\times\cdots\times I_d}$ as a sum of $r$ rank-one outer products $\mathcal{T}=\sum_{j=1}^ra_1^{(j)}\otimes\cdots\otimes a_d^{(j)}$, stored as $O(r\cdot d\cdot I_\mathrm{max})$ floats versus the $O(I_\mathrm{max}^d)$ of the dense format. The Tucker decomposition \cite{tucker1966} generalizes this by coupling mode-wise factor matrices through a core tensor, providing a more flexible but still compact representation. Tensor Train (TT) \cite{oseledets2011} decomposes the array as a chain of third-order cores whose sequential contraction recovers any entry in $O(d\cdot r^2)$ operations, where $r$ is the TT rank. The algebraic properties, computational algorithms and empirical performance of all three formats have been surveyed comprehensively \cite{grasedyck2013}. Their application to numerical linear algebra, high-dimensional quadrature and the discretization of parametric PDEs has been developed \cite{bachmayr2023}. In that context, tensor formats replace the exponentially large stiffness matrices that arise from full tensor-product discretizations with structured approximants with polynomial scaling cost in $d$ and $r$ \cite{khoromskij2012,hackbusch2012}.

Additive structure appears as the theoretical basis for generalized additive models (GAMs) \cite{hastie1990}. Approximants are represented as a sum of univariate, smooth component functions  $f(x)=\sum_{j=1}^d f_j(x_j)$. Each component is estimable from data by backfitting. GAMs are interpretable, tractable and efficient when the ground truth is, or is nearly, additive. Optimal minimax convergence rates for additive nonparametric estimators have been established \cite{stone1985}, making the point: the curse of dimensionality is circumventable when the true function belongs to a structured class -- in this case, additive. The rate of convergence is then determined by the smoothness of univariate components, not by $d$. Higher-order interaction models extend the additive framework by including $k$-way interaction terms, expanding the representable function class at the cost of increased parameter count. The theoretical properties of such models follow from classical ANOVA-style decompositions of the function space $L^2([0,1]^d)$ \cite{hoeffding1948, sobol1993}.

\subsection{Separable, tensorized and structured neural architectures}

The Kolmogorov-Arnold representation theorem \cite{kolmogorov1957} establishes that every continuous function $f:[0,1]^d\to\mathbb{R}$ can be written as a finite superposition of continuous univariate functions $f(x)=\sum_{q=1}^{2d+1}\Phi_q\left(\sum_{p=1}^d\phi_{q,p}(x_p)\right)$. This result provided early theoretical grounding for the expressive power of shallow networks \cite{cybenko1989,hornik1991}. Kolmogorov-Arnold networks (KANs) \cite{liu2025kan} make use of this theorem's constructions. As opposed to the fixed, nonlinear activation functions at nodes, KANs place learnable spline-parameterized univariate functions at the edges of a multilayer graph, yielding approximators that are more interpretable, parameter-efficient and empirically exhibit faster scaling laws than multilayer perceptrons of comparable expressive power. In the language of this work, the spline atoms of a KAN are univariate subatoms with $|S|=1$. The architecture nevertheless differs in a fundamental aspect. Namely, KAN composition follows the monolithic superposition structure of the Kolmogorov theorem: outer functions act on sums of inner functions, with no explicit control over interaction arity. The resulting architecture carries no rank parameter, no interaction tensor and no variational theory. KANs have been applied to PDEs in a PINN-adjacent collocation setting, but carry no variational theory connecting the architecture to the underlying function space.

Separable PINNs (SPINNs) \cite{cho2023} observe that representing the PINN solution as a sum of outer products of per-axis univariate networks -- a rank-$r$ CP ansatz evaluated pointwise -- dramatically reduces the cost of PDE residual evaluation. By exploiting automatic differentiation over the Kronecker-product structure of the separated collocation grid, SPINN achieves a reduction in residual computation cost from $O(C^d)$ to $O(C\cdot d\cdot r)$ per pass. This enables collocation point counts exceeding $10^7$ on a single GPU. Whilst the empirical gains are substantial, this separation is a computational device rather than an architectural design principle. The rank $r$ is a tuning parameter without theoretical status, there is no notion of interaction order and no variational framework connects the separated ansatz to the PDE's underlying function space. Thus, approximation-theoretic guarantees are entirely absent.

The hierarchical deep learning neural network (HiDeNN) \cite{saha2021hidenn} and interpolating neural network (INN) family \cite{park2025unifying} are this work's most direct antecedents, taking a complementary approach. Its structure directly mirrors finite element interpolation, reproducing partition-of-unity properties, nodal interpolation and automatic mesh refinement within a differentiable learning framework. The finite element heritage gives HiDeNN a connection to the classical variational setting, and makes outputs interpretable in mechanistic terms. HiDeNN-TD \cite{zhang2022hidenntd} extends this to the tensor-decomposed regime, separating the trial function over spatial and temporal coordinate axes, solving the resulting sequence of low-dimensional systems. C-HiDeNN-TD \cite{li2023chidenn_td,guo2024chidenn_td} achieves higher-order convergence via convolution patch functions and has been demonstrated in transfer learning \cite{mostakim2025inntl}. Building upon this, Ex-HiDeNN \cite{batley2025separable, saha2024dissertation,liu2026mci} constructs a symbolic surrogate by assigning univariate expressions to each learned coordinate axis. Across this family, the network graph is decomposed along coordinate axes rather than defining a functional class with explicit rank and interaction hyperparameters. No formulation provides well-posedness or approximation-theoretic guarantees over the full spatiotemporal--parametric domain. The SNA/VSNA generalizes this framework, providing the missing guarantees.

\subsection{Neural approximators for PDEs}

The idea of representing PDE solutions via neural network predates the deep learning era. Proposed as trial functions for boundary value problems, PDE residuals were penalized at interior collocation points and boundary conditions were encoded analytically in the ansatz \cite{lagaris1998}. Physics-informed neural networks (PINNs) \cite{raissi2019} systematized this approach. A single multilayer perceptron (MLP) is trained by minimizing a composite loss: a strong-form PDE residual, boundary, and initial condition losses and -- optionally -- measured data. Their theoretical basis has been substantially developed \cite{karniadakis2021, Cuomo2022}.

Operator learning generalizes the mapping from data to function. Rather than approximating a single solution, such architectures approximate the operator that maps input function spaces (forcing terms, boundary data, initial conditions) to solution spaces. DeepONet \cite{lu2021deeponet} separates the operator into branch and trunk networks, encoding the input function and output query point respectively. The Fourier neural operator (FNO) \cite{li2021fno} implements the operator in Fourier space, parameterizing a kernel integral operator by its spectral representation and evaluating it efficiently via the fast Fourier transform (FFT). Both architectures show strong performance on benchmark PDE families and applied engineering problems \cite{huang2023introduction}, serving as typical comparison points.

The structural gap these approaches share is the absence of any inductive bias reflecting separability or low-rank structure in the underlying problem. A standard PINN, DeepONet or FNO is a monolithic composition of linear maps and nonlinearities. They contain no prior favoring solutions that are approximately separable or that admit low-interaction representations. This implies that these architectures allocate parameters to model interactions of all orders, including those that contribute negligibly to the solution, resulting in commensurately large parameter counts. Crucially, these architectures do not connect to the Galerkin framework that underpins FEM and spectral methods. There is no analogue of Céa's lemma, no quasi-optimality bound, no mechanism for relating network architecture to approximation-theoretic properties of the trial space -- hyperparameter choices are arbitrary, with little predictable effect on solution accuracy. The VSNA is designed to close precisely this gap.

The gap these lines of work collectively expose is the absence of a unified architectural formalism that simultaneously
\begin{itemize}
    \item controls interaction order $k$ and tensor rank $r$ as explicit, theoretically meaningful hyperparameters;
    \item equips the resulting function class with universality and convergence guarantees rooted in classical approximation theory;
    \item embeds this class within a rigorous Galerkin variational framework over coupled spatiotemporal--parametric domains; and
    \item is applicable as a data-driven regression primitive and a PDE trial space.
\end{itemize}
The following section addresses all four simultaneously.

\section{Methods}
\label{sec:methods}

\subsection{Mathematical foundations of SNA}

\paragraph{Preliminaries}
For the present work, the ambient dimension is denoted $d\in\mathbb{N}$ and the domain $[0,1]^d$. An interaction hyperparameter $k\leq d$ is introduced, confining the maximum order of coordinate interaction. Denoting the index set $[d]:=\{1,\dots,d\}$, any subset $S\subseteq[d]$ with $|S|\leq k$, $x_S$ denotes the projection of $x$ onto $S$. A hyperparameter $r$ constrains interaction rank. Let $\rho:\mathbb{R}\rightarrow\mathbb{R}$ be an activation function. 

\paragraph{Foundational constituents} The architecture is constructed from learnable functional components, herein termed \emph{atoms}. Formally, an atom is defined as the learnable function $\phi^{(S)}:[0,1]^{|S|}\times\Theta_S\to\mathbb{R}$, parameterized by $\theta_S$. Interactions between these atoms are governed by an \emph{interaction object} $\mathcal{C}$; this is a collection of coefficients $c_S$ assigned to specific subsets of coordinates $S \subseteq [d]$, where $\operatorname{Supp}(\mathcal{C})=\{S\subseteq[d]\mid c_S\neq 0\}$. This object stores the abstract set of permissible coordinate interactions within the model. In particular, $\mathcal{C}$ admits a canonical embedding into a $k$-sparse, order-$d$ interaction tensor via the mapping $\mathcal{E}:\mathcal{C}\to\mathbb{R}^{d\times\dots\times d}$ with $\mathcal{E}(\mathcal{C})=\mathcal{T}$. The \emph{rank} of this interaction tensor $\mathcal{T}$ is strictly bounded by $r$, and $\mathcal{T}$ is nonzero only on entries indexed by subsets $S$ where $|S|\leq k$.

\paragraph{Separable Neural Architecture}
A Separable Neural Architecture (SNA), parameterized by $\Theta=\{\theta_S,c_S\}$, is a mapping $f:[0,1]^d\times\Theta\to\mathbb{R}$ defined as 
\begin{align}
    \boxed{f(x;\Theta)=\rho\left(\sum_{S\in\operatorname{Supp}(\mathcal{C})}c_S\phi^{(S)}(x_S;\theta_S)\right).}
\end{align}
This function represents an element of the functional class
\begin{align}
    \mathcal{F}_{k,r}=\left\{f(x, \Theta)\mid\operatorname{rank}(\mathcal{E}(\mathcal{C}))\leq r \textrm{ and } |S|\leq k,\forall S\in\operatorname{Supp}(\mathcal{C})\right\}.
\end{align}
The architectural interpretation of this definition is shown in Figure \ref{fig:sna}. The SNA operates through three stages:
\begin{inparaenum}[(i)]
    \item the ambient input vector $x$ is sliced into localized coordinate subsets $S$ of maximum arity $k$; 
    \item these subsets are evaluated in parallel by a dictionary of independent, learnable functional blocks (the atoms $\phi^{(S)}$); and 
    \item the atomic outputs are aggregated by weighted sum as prescribed by the interaction object $\mathcal{C}$, prior to passing through the global activation function $\rho$. 
\end{inparaenum}
By decoupling the functional complexity (assigned to low-dimensional atoms) from the interaction complexity (gated by a low-rank tensor budget $r$), the SNA remains compact.

\paragraph{Recovery of model families}
This definition generalizes several contemporary model families based on the interaction order $k$:
\begin{itemize}
    \item \textbf{Generalized additive models.} Recovered when $k=1$;
    \item \textbf{Pairwise interaction models.} Recovered when $k=2$;
    \item \textbf{Canonical Polyadic decomposed models.} Recovered when $k=|S|=d$. Atoms are products of univariate subatoms and $\operatorname{rank}(\mathcal{E}(\mathcal{C}))\leq r$.
\end{itemize}

\paragraph{Tensorized classes}

A subclass of SNAs, the tensor decomposition (TD) class is defined by setting $k=d$ and restricting $\operatorname{Supp}(\mathcal{C})$ to the single full-index subset $S=[d]$. The interaction object is embedded as a tensor decomposition with $\mathcal{C}=\mathcal{E}^{-1}(\mathcal{T})$ and rank $r$. An element of this class takes the form 
\begin{align}
    f(x;\Theta_{TD})=\rho\left(\sum_{j=1}^rc^{(j)}\phi^{(j)}(x; \theta^{(j)})\right),
\end{align}
where each rank component is evaluated by a multivariate atom $\phi^{(j)}$. This class specializes further into its most fundamental subclass, the Canonical Polyadic (CP)-class. Here atoms are restricted to be products of univariate subatoms $\psi^{(j)}_i$. An element of this class is written
\begin{align}
    f(x;\Theta_{CP})=\rho\left(\sum_{j=1}^rc^{(j)}\prod_{i=1}^d\psi_i^{(j)}(x_i; \theta_i^{(j)})\right).
\end{align}
Crucially, this structural restriction does not forfeit expressivity. Consider the functional class of CP-SNAs $\mathcal{F}^r_{CP}$ under identity activation $\rho(x)=x$. Then, if each subatom $\psi^{(j)}_i$ is continuous on the unit segment, the union over finite ranks $\mathcal{F}=\bigcup_{r=1}^\infty\mathcal{F}^r_{CP}$ is dense in $C\left([0,1]^d\right)$ with respect to the infinity norm $\|\cdot\|_\infty$. $L_p$-convergence, that $\mathcal{F}$ is dense in $L^p\left([0,1]^d\right)$ for any $p\geq1$, follows from the density of $C([0,1]^d)$ in $L^p([0,1]^d)$.

Since the CP class is contained within the TD class, it follows as corollary that TD-class SNAs are also dense in $C\left([0,1]^d\right)$. This includes the Tucker \cite{tucker1966} and Tensor Train \cite{oseledets2011} decompositions. Notably, then, every target function can be approximated arbitrarily well by some TD-class SNA with finite rank $r$.

\subsection{Variational Separable Neural Architectures}
\label{sec:VSNA}

\paragraph{Variational domains and trial spaces}

In extending this formalism to variational problems, let $\Omega \subset \mathbb{R}^n$ be an $n$-dimensional spatial domain, $[0, T]$ a temporal interval and $\mathcal{P} \subset \mathbb{R}^m$ a parametric space. The combined domain is the Cartesian product $\mathcal{X} = \Omega \times [0, T] \times \mathcal{P}$, with ambient dimension $d = n + 1 + m$. Coordinates are indexed $i \in \{1, \dots, d\}$ covering space, time and parameters. The separable trial space is the Hilbert tensor product $V=\bigotimes_{i=1}^dV^{(i)}$ where $V^{(i)}$ is an appropriate Hilbert space over the $i$-th coordinate. For a semi-linear form $a:V\times V\rightarrow\mathbb{R}$ and linear functional $\ell:V\rightarrow\mathbb{R}$, $u\in V$ is sought such that $a(u,v)=\ell(v)$ for all $v\in V$.

\paragraph{Variational Separable Neural Architecture}

A Variational Separable Neural Architecture (VSNA), parameterized by $\Theta$, is a trial function $u \in V$ defined as
\begin{equation}
    \boxed{u(x; \Theta) = \sum_{S\in\operatorname{Supp}(\mathcal{C})} c_S \phi^{(S)}(x_S;\theta_S)}
\end{equation}
where each atom $\phi^{(S)} \in V^{(S)} = \bigotimes_{i \in S} V^{(i)}$ respects the local variational structure of its coordinates. This trial function represents an element of the finite-dimensional approximation subspace,
\begin{equation}
    \mathcal{F}_{k,r} = \left\{ u(x, \Theta) \in V : \operatorname{rank}(\mathcal{E}(\mathcal{C})) \leq r,\ |S| \leq k,\forall S\in\operatorname{Supp}(\mathcal{C})\right\}.
\end{equation}

\paragraph{Variational tensorized classes}

The class of CP-VSNA trial functions is defined by:
\begin{align}
    \mathcal{F}_{r} = \left\{ u(x; \Theta) = \sum_{j=1}^r c^{(j)}\prod_{i=1}^d \psi_i^{(j)}(x_i; \theta^{(j)}_i) \;\middle|\; \psi_i^{(j)} \in V^{(i)},c^{(j)}\in\mathbb{R},\theta^{(j)}_i\in\Theta\right\}
\end{align}
given a learnable parameter set $\Theta$. Thus, $\mathcal{F}_{r} \subset V$ serves as the finite-dimensional Galerkin trial space.

\paragraph{Degenerate variational classes}
The CP-VSNA subsumes several established variational and operator-learning architectures as special cases under rank and interaction constraints on $[d]$.
\begin{itemize}
    \item \textbf{Variational DeepONet:} Partitioning coordinates as $[d]=\mathcal{I}_x\cup\mathcal{I}_\mu$, restricting the CP product to respect this bipartition, and minimizing a variational energy loss recovers the branch-trunk decomposition $u(x\,;\mu)=\sum_{j=1}^r b_j(\mu)t_j(x)$ of V-DeepONet \cite{goswami2022vdeeponet}. A distinction lies in that V-DeepONet operates in the operator-learning setting, mapping input functions to solutions. The VSNA, as stated, maps finite-dimensional parameter vectors. This makes the structural correspondence exact when the input is parameterized by a finite set of coordinates.
    \item \textbf{POD-Galerkin reduced basis:} Fixing the spatial subatoms $\psi_i^{(j)}(\cdot\,;\theta_i^{(j)})$ as precomputed POD modes and training only the parametric factors recovers the classical reduced-basis method \cite{rozza2008,benner2015} as a rank-$r$ VSNA with frozen spatial atoms.
    \item \textbf{PGD:} Greedy rank-one enrichment of the VSNA -- appending a single new CP term at each step and freezing prior terms -- recovers the proper generalized decomposition \cite{ammar2006pgd, Chinesta2011}. Rank-$r$ training strictly generalizes PGD by allowing all rank components to communicate.
\end{itemize}

\paragraph{Variational guarantees}

To establish the classical validity of this trial space, let $a: V \times V \rightarrow \mathbb{R}$ be a bounded and coercive bilinear form with coercivity constant $c_0 > 0$ and boundedness constant $c_1 > 0$. Denoting a linear functional $\ell \in V^*$, and fixing the basis subatoms $\psi_i^{(j)}$, the VSNA formalism satisfies four core variational guarantees:
\begin{itemize}
    \item \textbf{Well-posedness:} The Galerkin approximation, $a(u_{r}, v_{r}) = \ell(v_{r})$ for all $v_{r} \in \mathcal{F}_{r}$, admits a unique solution $u_{r} \in \mathcal{F}_{r}$.
    \item \textbf{Quasi-optimality:} Let $u \in V$ be the unique solution to the exact weak problem. The VSNA Galerkin solution $u_{r}$ is quasi-optimal, bounded strictly by the best approximation within the trial space: $\|u - u_{r}\|_V \leq \frac{c_1}{c_0} \inf_{v_{r} \in \mathcal{F}_{r}} \|u - v_{r}\|_V$.
    \item \textbf{Convergence:} If each univariate subatom family $\psi^{(j)}_i(\cdot~; \theta^{(j)}_i)$ is dense in $V^{(i)}$, then $\bigcup_r\mathcal{F}_{r}$ is dense in $V$ with respect to the Hilbert norm $\|\cdot\|_V$. Consequently, as the interaction rank $r \to \infty$, the approximation error $\varepsilon_r \to 0$, ensuring VSNA Galerkin solutions converge to the exact solution $u \in V$.
    \item \textbf{Stability:} The VSNA Galerkin solution $u_{r}$ satisfies the absolute stability bound $\|u_{r}\|_V \leq \frac{1}{c_0} \|\ell\|_{V^*}$, where the dual norm is defined as $\|\ell\|_{V^*} = \sup_{v\in V\setminus\{0\}} \frac{\ell(v)}{\|v\|_V}$.
\end{itemize}

Under the Lax-Milgram hypotheses, the VSNA trial space $\mathcal{F}_r$ inherits the complete Galerkin apparatus: it is well-posed, quasi-optimal, stable and convergent as $r\to\infty$. Further details on this has been added to \ref{App:A}.

\subsection{Implementation}

The CP-class SNA with subatoms parameterized by B-splines -- introduced under the name KHRONOS \cite{batley2025khronos,sarker2026khronos} -- instantiates this framework as follows.  It adopts identity activation  ($\rho(x)=x$) and unit modal weights ($c^{(j)}\equiv1$). For a spline basis of order $P$ over $C_i$ interior cells in each dimension $i$, each subatom takes the form
\begin{align}
    \psi_i^{(j)}(x_i;\theta^{(j)}_i)=\sum_{c=1}^{C_i+P}\alpha^{(j)}_{i,c}B^{P}_c(x_i).
\end{align}
The extended index $C_i+P$ accounts for domain-exterior ``ghost'' cells required to preserve partition of unity \cite{piegl1997nurbs}. This instantiation -- predictive -- is trained directly from data by minimizing a supervised loss over the subatom parameters $\{\alpha^{(j)}_{i,c}\}$. The commensurate VSNA instance replaces the data-driven training mode with a variational physics loss as established in Section \ref{sec:VSNA}. The resulting variational problem is solved via minimizing the weak residual of the governing operator.

\paragraph{Alternating least squares}

Direct solution of resulting global Galerkin systems $Au=b$ over $d$-dimensional domains requires assembling an operator of size $O(C^d \times C^d)$, where $C$ is the number of per-dimension degrees of freedom. For the high-dimensional spatiotemporal--parametric problem considered, this global assembly is computationally intractable. By enforcing the Canonical Polyadic (CP) structure on the trial space, the globally nonlinear (due to multilinearity of CP parameterization) Galerkin projection is solved via a sequence of linear updates. 

Fixing all dimensions other than the $\ell$-th and solving for $\theta^{(\ell)}$ alone, the local ALS \cite{holtz2012als, kolda2009} system is written
\begin{equation}
    \left({A}^{(\ell)} + \lambda I \right) \mathrm{vec}(\theta^{(\ell)}) = b^{(\ell)}
\end{equation}
where $\lambda$ is a small Tikhonov regularization parameter introduced to ensure strict positive-definiteness of the local operator. 

Crucially, the local stiffness matrix $A^{(\ell)}$ and load vector $b^{(\ell)}$ are assembled without ever constructing the global operators. Instead, they are computed exactly via Kronecker products of one-dimensional integral matrices from the active dimension and the contracted interaction weights from the fixed dimensions. 

Physical constraints are imposed explicitly. Spatial Dirichlet boundary conditions and temporal initial conditions are strongly enforced on the relevant univariate subatoms prior to each local solve. This is possible by lifting the initial condition $u(\cdot,t=0)=u_0$ to write $u=u'+u_0$ and solving for $u'$ with the temporal subatom initially respecting a homogeneous Dirichlet condition. To prevent numerical underflow or overflow across the high-order tensor products, the active subatoms are re-normalized after each update, with the scalar magnitudes symmetrically absorbed into the remaining fixed dimensions.

This iterative procedure sweeps through all $d$ dimensions sequentially until the relative change in the global residual norm falls below a specified tolerance. This process is systematized in Algorithm \ref{alg:als_vsna}. By operating entirely within the factorized latent space, the tensor-native formulation reduces the computational complexity of solving the high-dimensional PDE from $O(C^{3d})$ to $O(I_\mathrm{max} \cdot d \cdot r^3 \cdot C^3)$ for direct solves, reducing further to $O(I_\mathrm{max} \cdot d \cdot r^2 \cdot C)$ under iterative solve of each system \cite{holtz2012als}. $I_\mathrm{max}$ is the number of ALS iterations. In either case, this reduces the exponential scaling bottleneck of traditional grid-based discretizations to polynomial in $d, r$ and $C$.

\begin{algorithm}[H]
\caption{Tensor-Native Alternating Least Squares (ALS) for VSNA Field Solution}
\label{alg:als_vsna}
\begin{algorithmic}[1]
\Require Governing weak operators $\mathcal{A}$, source terms $b$, interaction rank $r$, basis resolution $C$, tolerance $\tau$.
\Ensure Fully optimized coordinate subatom coefficients $\Theta = \{\theta^{(1)}, \dots, \theta^{(d)}\}$.
\State Initialize subatom coordinate modes $\Theta$ and strongly enforce boundary/initial condition lifting ($u \leftarrow u' + u_0$).
\State Compute initial global weak residual norm $e_0$ via multi-dimensional interaction contractions.
\While{relative error $\Delta e > \tau$ and $k \leq I_{\max}$}
    \For{dimension $\ell = 1$ to $d$}
        \State Freeze parameter matrices for all inactive dimensions: $\Theta_{\text{fixed}} \leftarrow \{\theta^{(i)}\}_{i \neq \ell}$.
        \State Contract fixed dimensions over 1D quadratures using element-wise Hadamard products.
        \State Assemble local normal equations matrix $A^{(\ell)}$ and right-hand side vector $b^{(\ell)}$ via Kronecker products of active 1D operators.
        \State Solve the local linear system for the active mode coefficients: $\mathrm{vec}(\theta^{(\ell)}) \leftarrow \left(A^{(\ell)}\right)^{-1} b^{(\ell)}$.
    \EndFor
    \State Compute updated global weak residual norm $e_k$ via operator and source interaction pairings.
    \State $k \leftarrow k + 1$
\EndWhile
\end{algorithmic}
\end{algorithm}

\paragraph{Global optimization}

ALS is the preferred training strategy for variational problems in moderate-to-high dimension, where Kronecker structure enables exact operator assembly. For low-dimensional problems, or data-driven regression where the loss landscape is smooth and parameter count modest, standard gradient-based optimizers -- SGD, Adam, L-BFGS -- applied globally to all subatom coefficients simultaneously are a viable and often faster alternative. Global optimizers allow all rank components to communicate throughout training. For an SNA with $N$ total parameters, these approaches incur the usual costs: $O(N)$ for Adam \cite{kingma2015adam}, $O(mN)$ for L-BFGS \cite{liu1989lbfgs} with $m$ the history length, and $O(N^3)$ for Newton.

\section{Numerical Validation and Scaling Behavior}
\label{sec:results}

The numerical studies presented in this section are organized as follows. Section \ref{sec:convergence} validates the theoretical convergence of the VSNA on a two-dimensional Poisson problem to confirm $p+1$ rates for B-spline subatoms of degree $p$ and super-algebraic convergence for Fourier bases. Section \ref{sec:nonlinear} extends the VSNA to the nonlinear setting via the inviscid Burgers equation to demonstrate robustness beyond coercive bilinear theory. Section \ref{sec:6dVSNA} solves a six-dimensional spatiotemporal--parametric advection--diffusion equation, validating the VSNA as a world model over the full coupled manifold and quantifying the empirical scaling of the VSNA. 

The VSNA is thus validated across the triad of canonical PDE characters: elliptic (Poisson), parabolic (advection--diffusion) and hyperbolic (Burgers).

The SNA is further validated as a data-driven neural primitive in regression and classification. Section \ref{sec:sobol} benchmarks the CP-class SNA on a 20-dimensional regression problem against competing architectures. Section \ref{sec:imagenet} presents a scaling study for the CP-class SNA tasked with classifying images from the ImageNet dataset \cite{deng2009imagenet} pre-encoded by CLIP \cite{radford2021clip}, Section \ref{sec:mnist} presents a scaling study on pretrained latent spaces on the MNIST dataset.

\subsection{Convergence behavior}
\label{sec:convergence}

The theoretical convergence guarantees established in Section \ref{sec:VSNA} are validated on a two-dimensional Poisson equation $-\nabla^2 u = f$ on $\Omega=[0,1]^2$ with homogeneous Dirichlet boundary conditions $u|_{\partial\Omega}=0$. The manufactured solution is
\begin{equation}
    u^*(x,y) = \exp\!\left(-\frac{(x-\tfrac{1}{2})^2
    +(y-\tfrac{1}{2})^2}{2\sigma^2}\right), \quad \sigma=0.05,
\end{equation}
from which the forcing term $f=-\nabla^2 u^*$ is derived analytically. A rank-$2$ CP-class VSNA is trained by minimizing the weak residual via ALS. The $L^2$ error is computed as $\|u_r - u^*\|_{L^2(\Omega)}$ via a tensor-product Gauss-Legendre quadrature grid over $\Omega$. For the Fourier sine basis the number of retained modes per dimension is denoted $C$, as a direct analogue to ``cells''.

Figure \ref{fig:convergence} shows $L^2$ error against number of cells $C$ for B-spline subatoms of degree $p\in\{1,2,3,4\}$, and Fourier sine subatoms. Spline convergence rates of $p+1$ are achieved in each case, in agreement with approximation theory. The Fourier sine basis exhibits super-algebraic convergence, as expected. The choice of subatom basis is thus problem-dependent, with B-splines preferred for problems with limited regularity, discontinuities or non-periodic boundaries. 

\begin{figure}[h!]
    \centering
    \includegraphics[width=0.8\linewidth]{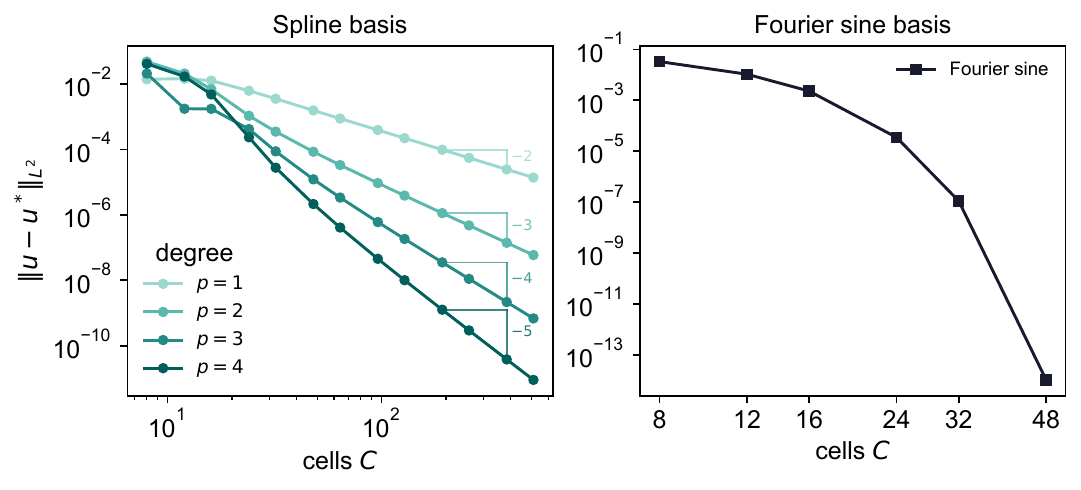}
    \caption{\textbf{VSNA convergence behaviour.} $L^2$ error versus cells $C$ for B-spline subatoms of degree $p\in\{1,2,3,4\}$ (\emph{left}) and Fourier sine subatoms (\emph{right}).}
    \label{fig:convergence}
\end{figure}

\subsection{Nonlinear extension}
\label{sec:nonlinear}

Moving beyond the current theoretical guarantees, the VSNA is evaluated on nonconvex, non-coercive nonlinear equations. Convergence behavior in this regime is demonstrated on the one-dimensional viscous Burgers equation
\begin{align}
    \partial_tu+u\partial_x u=\nu\partial_{xx}u,\quad(x,t)\in[0,1]^2.
\end{align}
Here $\nu$ denotes kinematic viscosity, and the equation is solved subject to homogeneous Dirichlet boundary conditions and initial condition $u(x,t)=u_0$. Two distinct physical regimes are analyzed: the viscous regime ($\nu=0.01$), characterized by a steepening internal layer, and inviscid regime ($\nu=0$) where the field may form shocks. 

The trial space is constructed via a rank-$r$ CP-class VSNA, and strongly enforces the initial condition via lifting,
\begin{equation}
u_r(x,t; \Theta) = u_0(x) + \sum_{i=1}^{r} \alpha_i(x; \theta_i^{(\alpha)}) \gamma_i(t; \theta_i^{(\gamma)}).
\end{equation}
The spatial subatoms $\alpha_i$ are expanded in a basis of $C$ Fourier sine modes, natively satisfying the Dirichlet boundary conditions, and the temporal subatoms $\gamma_i$ are parameterized via $C$ shifted Chebyshev polynomials. The model parameters $\theta=[\alpha,\gamma]^T\in\mathbb{R}^{2rC}$ are determined by minimizing the squared $L^2([0,1]^2)$ weak residual functional $\mathcal{L}(\theta)=\|R(u_r)\|^2_{L^2([0,1]^2)}$. The advection term $u\partial_x u$ induces a non-convex trilinear form, placing this problem outside not only the coercive but the bilinear setting of the established theory.

The weak residual,
\begin{align}
\label{eq:weak}
    \mathcal{R}\{\hat{u}, v\} = \int_0^1\int_0^1 v(
    \partial_t \hat{u} + \hat{u}\,\partial_x \hat{u} -\nu\partial_{xx}\hat u)~dx~dt = 0,
    \quad \forall v \in V,
\end{align}
is minimized over a finite test space. The separable structure of the trial space reduces \eqref{eq:weak} to contractions over one-dimensional quadratures to avoid grid-based assembly. It is notable that this reduction is agnostic to the order of the resulting multilinear form: whether the weak formulation is bilinear, as in linear PDE theory, or trilinear, as here, the trial space accordingly decomposes each integral into the same class of one-dimensional atom contractions. The inherent cost is in the proliferation of cross terms growing as $O(r^n)$ in the trial space rank $r$ and form order $n$. 

\begin{figure}[!h]
    \centering
    \includegraphics[width=\linewidth]{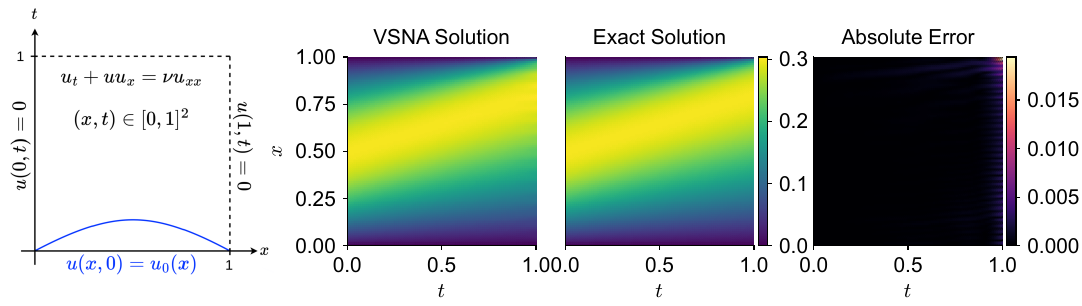}
    \caption{\textbf{VSNA solution of an inviscid Burgers equation.} The nonlinear initial-boundary-value problem schematic (\emph{left}). Filled contours of the recovered solution $u_\mathrm{VSNA}(x,t)$ (\emph{center-left}) and exact solution $u(x,t)$ (\emph{center-right}). Pointwise absolute errors $|u_\mathrm{VSNA}(x,t)-u(x,t)|$ (\emph{right}).}
    \label{fig:burgers}
\end{figure}

\paragraph{The pre-shock inviscid limit}

In the inviscid setting, $\nu=0$, structural dissipation is entirely removed. The weak residual simplifies to
\begin{align}
    \mathcal{R}\{\hat{u}, v\} = \int_0^1\int_0^1 v(
    \partial_t \hat{u} + \hat{u}\,\partial_x \hat{u} )~dx~dt = 0,
    \quad \forall v \in V,
\end{align}
and is minimized over a finite test space via L-BFGS. The result obtained with a CP-class VSNA of rank $r=12$ and $C=12$ cells -- here, Fourier modes in space, Chebyshev modes in time -- is shown in Figure \ref{fig:burgers}. The VSNA accurately recovers the smooth pre-shock solution, with absolute errors below 2\% throughout the domain. The spatially ringing residual structure visible as $x,t\rightarrow1$ is consistent with the convergence of characteristics at the domain corner at an impending shock.

\paragraph{Viscous setting} To provide a definitive metric for convergence, solutions are compared against an exact reference solution $u_\mathrm{exact}(x,t)$, computed via the analytical Cole--Hopf transform. The initial condition is sinusoidal, $u_0=\sin(2\pi x)$, which the VSNA solves by a two-phase optimizer. The first phase warms up with L-BFGS (history $m=20$, strong Wolfe line search) until the weak residual falls below $10^{-4}$. The second phase applies an exact Levenberg--Marquardt Newton method: the full Hessian is materialized via batched Hessian--vector products, factored by Cholesky. Rank continuation initializes each rank from the converged solution at the previous rank via CP mode duplication with additive noise. 

Convergence of the nonlinear approximation under basis and rank refinement is quantified in Figure \ref{fig:nonlinear_convergence}. Along individual rank-isolines, refining the Fourier-Chebyshev basis from $C=4$ to $C=128$ yields a super-algebraic reduction in $L^2$ error, spanning almost seven orders of magnitude. However, this basis refinement plateaus if the network lacks sufficient rank capacity; for example, the $r=4$ rank-isoline saturates prematurely at an error floor of $\approx 3\times10^{-3}$. However, rank refinement beyond $r=32$ yields no noticeable gains, suggesting that approximation error is basis-limited in this regime. Together, these results demonstrate that the nonlinear VSNA retains spectral convergence behavior despite the absence of the coercive bilinear structure underpinning the theory of Section \ref{sec:VSNA}.

\begin{figure}[!h]
    \centering
    \includegraphics[width=\linewidth]{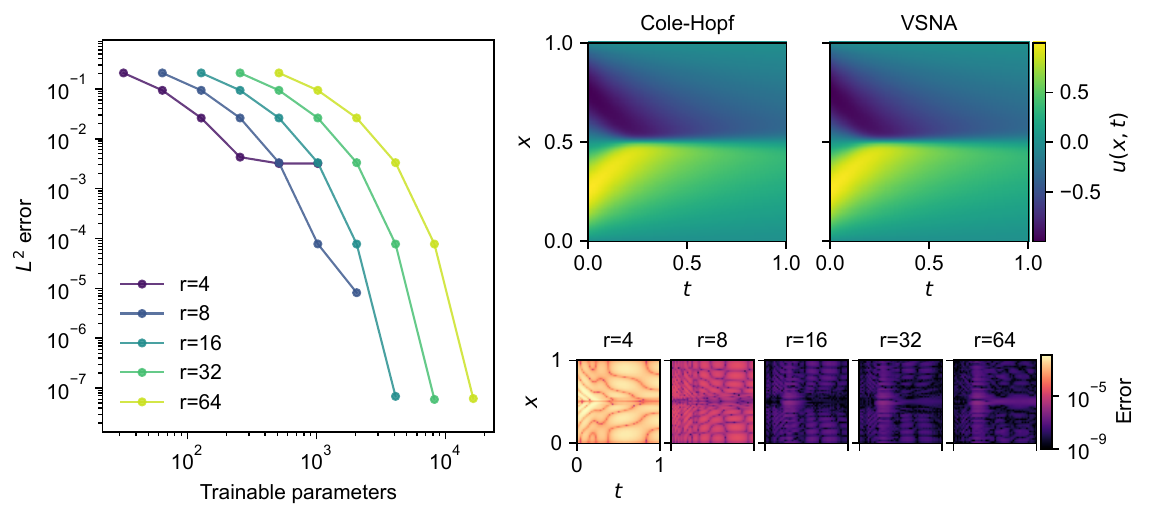}
    \caption{\textbf{Nonlinear VSNA convergence study.} Parametric scaling frontier of $L^2$ error against total trainable parameters (\emph{left}), traced along rank-isolines $r\in\{4,8,16,32,64\}$ with increasing resolution $C\in\{4,8,16,32,64,128\}$. The exact solution, derived via Cole-Hopf transformation, and VSNA approximation (from $r=128,C=32$)(\emph{top-right}) and pointwise absolute error maps for the best solution for each of $r=4,8,16,32,64$ (\emph{bottom-right}).}
    \label{fig:nonlinear_convergence}
\end{figure}

\subsection{Solving high-dimensional spatiotemporal--parametric PDEs}
\label{sec:6dVSNA}

\begin{figure}[!b]
    \centering
    \includegraphics[width=\linewidth]{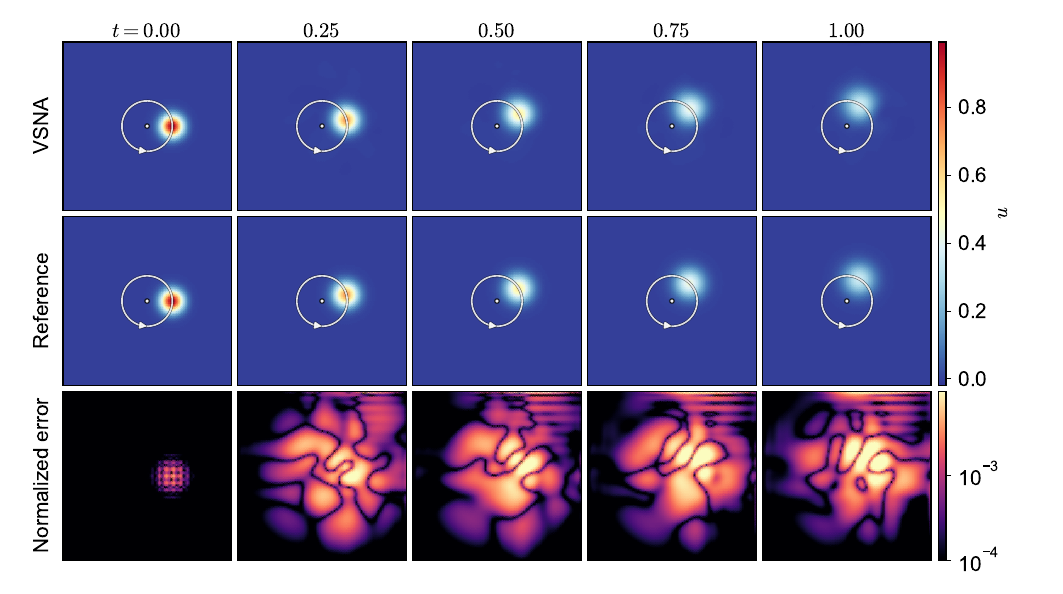}
    \caption{\textbf{Normalized errors.} Spatiotemporal evolution of the field for fixed $\omega=\frac \pi 3$ and $D=0.001$. The top and middle rows compare the VSNA solution with the reference, whilst the bottom row shows the normalized pointwise error $\frac{|u_\mathrm{VSNA}-u_\mathrm{ref}|}{\max|u_\mathrm{ref}|}$.}
    \label{fig:relative}
\end{figure}

The VSNA instance examined is of the same structure as the prior section -- CP-class -- but interpreted variationally as forming a finite-rank trial space over the spatiotemporal--parametric domain. An example six-dimensional spatiotemporal--parametric advection--diffusion system is considered:
\begin{align}
    \frac{\partial u}{\partial t}+\boldsymbol{U}\cdot\nabla u-D\nabla^2u=0.
\end{align}
The field $u$ evolves over spatial coordinates $(x,y,z)\in[0,1]^3$ with homogeneous Dirichlet boundary conditions, as well as time $t\in[0,1]$, angular velocity $\omega\in[0,\tfrac \pi 3]$ and diffusivity $D\in[0.001, 0.01]$. The motion of an initial Gaussian plume is driven by a two-dimensional solid-body rotating wind $\boldsymbol{U}=[-\omega(y-\tfrac1 2),\omega(x-\tfrac1 2),0]^T$. In physical application, such a system might model the
transport and dissipation of a scalar quantity -- energy, aerosols or pollutants -- within a rotating fluid domain \cite{vallis2017atmospheric}.

\begin{figure}[!t]
    \centering
    \includegraphics[width=0.8\linewidth]{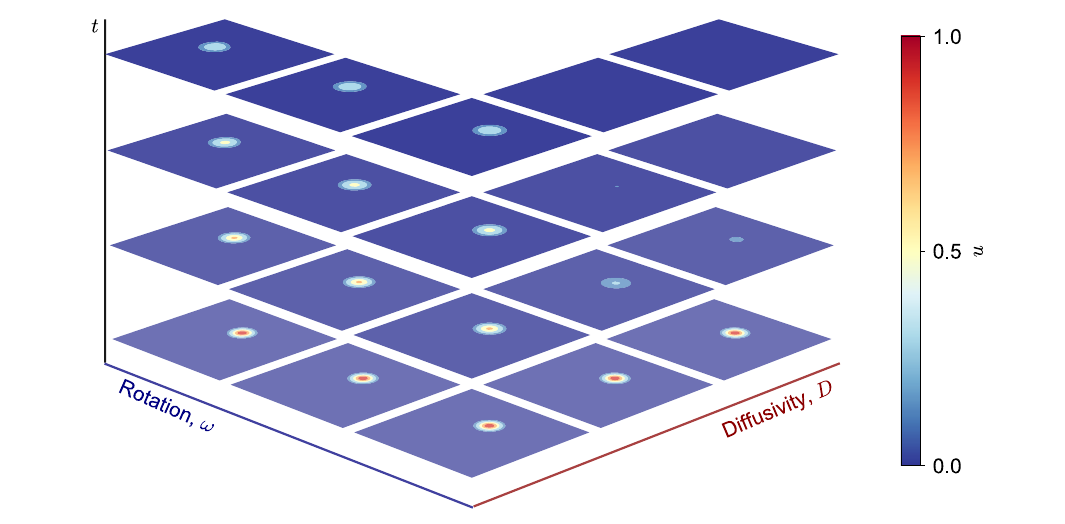}
    \caption{\textbf{Solution manifold.} The six-dimensional spatiotemporal--parametric advection-diffusion field learned by the VSNA. Stacked $x$-$y$ spatial slices across time $t$ are shown across rotation-diffusivity $\omega$-$D$ parameter space, illustrating recovery of the complete solution manifold within a single global representation.}
    \label{fig:manifold}
\end{figure}

Because the full six-dimensional solution manifold is captured as a separable field over all of space, time and parameters, the continuous field can be queried at any location across space, time \emph{and} parameter space. Whereas a classical FEM or standard PINN approach would require a full re-solve for each desired parameter combination, the VSNA provides the full space-time field, queried in milliseconds. Fig. \ref{fig:relative} illustrates representative two-dimensional spatial slices of the learned six-dimensional manifold at $\omega=\frac \pi 3$ and $D=0.001$. 

The VSNA's prediction is compared against the free-space Gaussian solution, which satisfies the governing equation exactly but the Dirichlet conditions only approximately. This proxy takes the form
\begin{align}
    u_\mathrm{ref}(x,y,z,t;\omega,D)&=G_x(x,t;\omega,D)G_y(y,t;\omega,D)G_z(z,t;\omega,D),\\
    G_\xi(\xi,t)&=\frac{\sigma_\xi}{\sqrt{\tau_\xi(t)}}\exp\left(-\frac{(\xi-\mu_\xi(t))^2}{2\tau_\xi(t)}\right),\tau_\xi(t)=\sigma_\xi^2+2Dt,\\
    \mu_x & =\tfrac 1 2 + \cos(\omega t)(\mu_{x,0}-\tfrac 1 2)-\sin(\omega t)(\mu_{y,0}-\tfrac 1 2), \\
    \mu_y & =\tfrac 1 2 + \sin(\omega t)(\mu_{x,0}-\tfrac 1 2)+\cos(\omega t)(\mu_{y,0}-\tfrac 1 2), \\
    \mu_z&=\mu_{z,0}.
\end{align}
Normalized pointwise errors $\frac{|u_\mathrm{VSNA}-u_\mathrm{ref}|}{\max|u_\mathrm{ref}|}$ are shown. The recovered field reproduces both rotational transport and diffusive spread -- albeit mild at this $D$ -- with high fidelity across time. Errors remain smooth and spatially structured. Having established that the CP-class VSNA recovers the coupled spatiotemporal--parametric dynamics, the natural question is how this accuracy scales with computational resources.

Fig. \ref{fig:scaling} quantifies approximation accuracy under joint refinement of rank $r$ and resolution $C$. Along fixed-rank trajectories, the error decreases approximately as $C^{-4}$ -- as is expected with cubic B-splines -- but saturates once rank capacity is reached. This combined effect produces an efficient frontier sustained across four orders of magnitude in trainable parameters $N$. This frontier follows an empirical scaling $\|e\|_{L^2}\approx0.24N^{-0.68}$, consistent with the theoretical convergence rate of $-\frac p d =-\frac4 6$ for cubic B-splines in six dimensions. 

\begin{figure}[!t]
    \centering
    \includegraphics[width=\linewidth]{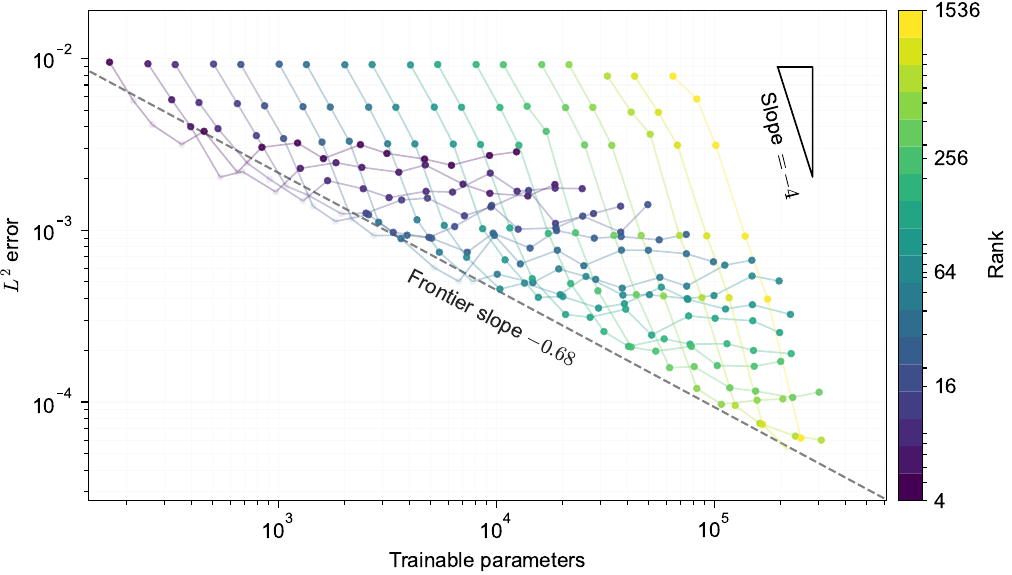}
    \caption{\textbf{VSNA scaling.} Approximation ($L^2$) error versus trainable parameters for the same system under refinement of rank $r$ and resolution $C$. Rank-isolines are connected and colour-coded. Along rank-isolines, errors decrease with resolution at slope $=-4$ before saturating at the rank capacity limit. Across ranks, an efficient frontier emerges (fitted slope $\approx-0.68$ in log-log space), sustained across four orders of magnitude in parameter count.}
    \label{fig:scaling}
\end{figure}

The improved intercept compresses the parameters needed to achieve a target error by six orders of magnitude compared to high-dimensional FEM. At a representative accuracy of $\|e\|_{L^2}=10^{-2}$, the VSNA requires $168$ parameters. For a rigorous computational comparison, a finite element discretization of the six-dimensional domain is constructed at $C=6$ cells per dimension. This yields $6^6=46{,}656$ degrees of freedom. A single solve at this resolution requires 434 seconds on an NVIDIA H200 GPU, requiring 0.084KWh of energy.

To match the precision of the high-resolution VSNA solve $C=32, r=1024$, with an error of $5.3\times10^{-5}$, an equivalent FEM grid requires an extremely dense discretization. Scaling the FEM baseline to a resolution-matched configuration of $C=94$ elements per dimension yields $6.8\times 10^{11}$ global degrees of freedom. This expands to a projected walltime of $4.5\times10^{16}$ years with a direct solver. The energy expenditure associated with a solve this large is on the order of $2.7\times 10^{20}$KWh.

In contrast, the VSNA resolves the entire parameter manifold in a walltime of 111s, consuming just 0.022KWh. Even when compared to a single measured low-resolution FEM evaluation, the VSNA achieves a lower approximation error whilst cutting computational walltime and energy expenditure by $4\times$. Compared to a comparably accurate FEM full-grid solve, the VSNA is $10^{22}$ faster and expends that much less energy. These results are summarized in Table \ref{tab:energy}.

\begin{table}[h!]
\centering
\caption{\textbf{Computational comparison: VSNA versus six-dimensional FEM.} The $C=94$
projection follows from the $O(C^{18})$ complexity of direct solvers.}
\label{tab:energy}
\begin{tabular}{lcccc}
\toprule
Method & DoFs / Parameters & $\|e\|_{L^2}$ & 
Wall time & Energy \\
\midrule
FEM $C=6$ (measured) & $36{,}750$ & 
$1.3\times10^{-2}$ & $434\,\mathrm{s}$ & 
$0.084\,\mathrm{kWh}$ \\
FEM $C=94$ (projected) & $6.8\times10^{11}$ & 
$5.3\times10^{-5}$ & 
$4.5\times10^{16}\,\mathrm{years}$ & 
$2.7\times10^{20}\,\mathrm{kWh}$ \\
\midrule
VSNA ($C=4, r=4$) & $168$ & $9.5\times10^{-3}$ & 
$8\,\mathrm{s}$ & $1.6\times10^{-3}\,\mathrm{kWh}$ \\
VSNA ($C=32, r=1024$) & $215{,}040$ & 
$5.3\times10^{-5}$ & $111\,\mathrm{s}$ & 
$0.022\,\mathrm{kWh}$ \\
\bottomrule
\end{tabular}
\end{table}

\subsection{20-dimensional regression benchmark}
\label{sec:sobol}

The CP-class SNA is evaluated against Random Forest (RF) \cite{breiman2001random}, XGBoost \cite{chen2016xgboost} and a multilayer perceptron (MLP) on a regression benchmark specifically designed to test separable structure exploitation. Model complexity is increased until validation $R^2=0.999$ to provide a consistent saturation point for comparison.

The Sobol-G function,
\begin{align}
    u(p)&=\prod_{i=1}^{20}\frac{|4p_i -2|+a_i}{1+a_i},\\
    a_i&=
    \begin{cases}
        0&\quad\textrm{for }i=1,\dots,5\\
        \frac32&\quad\textrm{for }i=6,\dots,10\\
        4&\quad\textrm{for }i=11,\dots,20
    \end{cases},\\
    p&\in[0, 1]^{20},
\end{align}
is itself fully separable: a product of univariate terms, and therefore lies precisely within the CP-class representation space -- a product of univariate terms is precisely a rank-one CP function. This makes it a direct empirical test of the central claim of this work -- that where separable structure exists, the SNA is the natural primitive to exploit it. Outputs are corrupted by additive noise $\epsilon\sim\mathcal{N}(0,0.01^2)$ and evaluated at 100,000 LHS points. Results are given in Table \ref{tab:sobol}. RF, XGBoost and the MLP all saturate far below the target accuracy; this is not for lack of parameters, but misaligned inductive biases. KHRONOS reaches $R^2=0.9994$ with 1,560 parameters, trained in 5.1 seconds, directly exploiting structure the other models cannot see.

\begin{table}[!t]
\centering
\caption{Regression benchmark on a noisy 20D Sobol-G Function.}
\label{tab:sobol}
\begin{tabular}{lcccc}
\toprule
Metric & Random Forest & XGBoost & MLP & KHRONOS \\
\midrule
Trainable parameters & 8,849,208 & 239,564 & 13,569 & 1560\\
Training time (s)       & 5.4 & 5.6 & 66 & 5.1\\
Test MSE& $4.6\times 10^{-4}$ & $3.3\times 10^{-5}$ & $1.4\times 10^{-4}$ & $6.8\times 10^{-7}$ \\
Test $R^2$              & 0.5565 & 0.7312 & 0.8788 & 0.9994\\
\bottomrule
\end{tabular}
\end{table}

\subsection{Empirical scaling in high-dimensional embedding spaces}
\label{sec:imagenet}

\begin{figure}[!b]
    \centering
    \includegraphics[width=\linewidth]{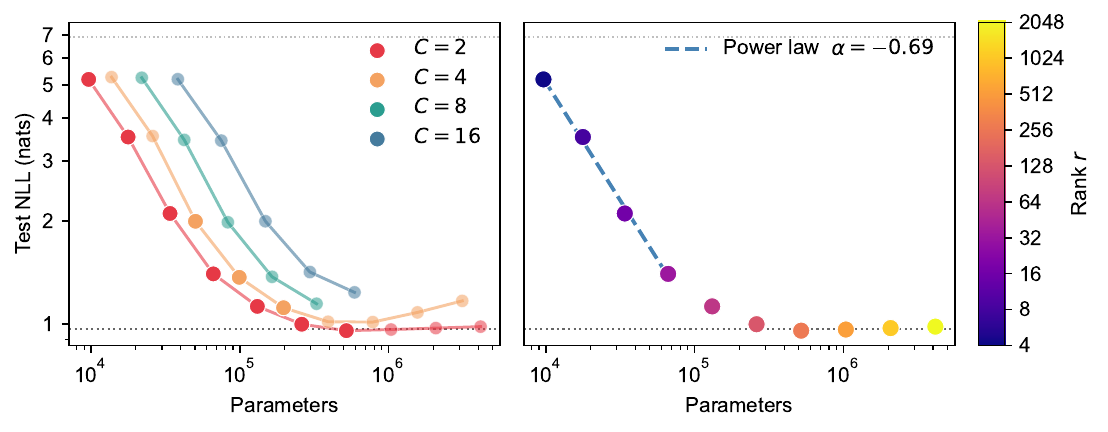}
    \caption{\textbf{Empirical scaling behavior of the SNA on ImageNet CLIP embedding space.} (\emph{left}) Total trainable parameters versus negative log-likelihood (NLL) under rank refinement for fixed $C=2,4,8,16$. (\emph{right}) Isolated power-law scaling fit for the Pareto-dominant $C=2$ configuration.}
    \label{fig:imagenet}
\end{figure}

To evaluate the expressive capacity and scaling behavior of the SNA in high-dimensional supervised learning, the architecture is deployed to map the continuous embedding space generated by a pretrained Contrastive Language-Image Pre-training (CLIP) \cite{radford2021clip} vision encoder on the ImageNet dataset \cite{deng2009imagenet}. This dataset is a large-scale computer vision classification task comprised of 1.2 million training images distributed across $K=1000$ distinct object categories. In particular, the embedding space is generated by CLIP ViT-B/32 with 512-dimensional embeddings $x\in\mathbb{R}^{512}$. These are normalized and bounded into $[0,1]^{512}$ via a sigmoid operator $\bar x=\sigma(\text{LN}(x))$ which serve as inputs to the network, then trained to predict image class labels. 

The SNA is extended to classification over $K$ discrete labels with a multi-output interface. This proceeds by preserving rank outputs without summing, retaining the uncontracted vector $\mathcal{R}(x)\in\mathbb{R}^r$. The conditional probability $\mathcal{P}(y=c\mid x)$ for class $c\in\{1,\dots,K\}$ is parameterized via standard softmax classification $\mathcal{P}(y = c \mid x) =\operatorname{Softmax}(\hat y_c)$. Here, $\hat y_c$ is composed of learned linear combinations of rank outputs $\mathcal{R}_j$ after normalization $\hat{y}_c = \sum_{j=1}^r w_{c,j} \text{LN}(\mathcal{R}(x)_j)$.

The empirical scaling behavior under joint variation of the atom resolution $C$ and rank $r$ is detailed in Figure \ref{fig:imagenet}. The left panel illustrates the effect of rank refinement for fixed basis resolution $C$. In this problem, since contrastive training encourages linear organization of the embedding space, increasing nonlinearity under resolution refinement is strictly parameter-inefficient. $C=2$ is Pareto-dominant, and is illustrated in the right panel. In the pre-asymptotic regime, the architecture exhibits a strict power law scaling over an order of magnitude with empirical convergence rate of $\alpha=-0.69$. However, as the parameter count grows past $10^5$, the convergence trajectory undergoes a phase transition, flattening abruptly over all resolution cases. The model's performance is strictly constrained by the artificial entropy floor of lossy compression imposed by CLIP.

\subsection{Dimension-dependent scaling in latent manifolds}
\label{sec:mnist}

Scaling is further tested in the case of a dimensionally widening latent bottleneck. Latent spaces, generated by a variational autoencoder of increasing dimensionality $d=\{4,8,16,32\}$ are trained on the MNIST dataset. The resulting test negative log-likelihood (NLL) scaling profiles across this regime are illustrated in Figure \ref{fig:mnist}. In the left panel, NLL drops for each $d$ until it encounters an aleatoric uncertainty floor inherent to lossy compression into a latent manifold. For every latent width $d$, the network demonstrates a pre-asymptotic power-law convergence regime; the decay rate appears similar. In the right panel, a power-law fit to this regime reveals an exponent of $\alpha=-1.18$, before tailing off to the entropy floor of $\approx0.09$ nats.

\begin{figure}[!t]
    \centering
    \includegraphics[width=\linewidth]{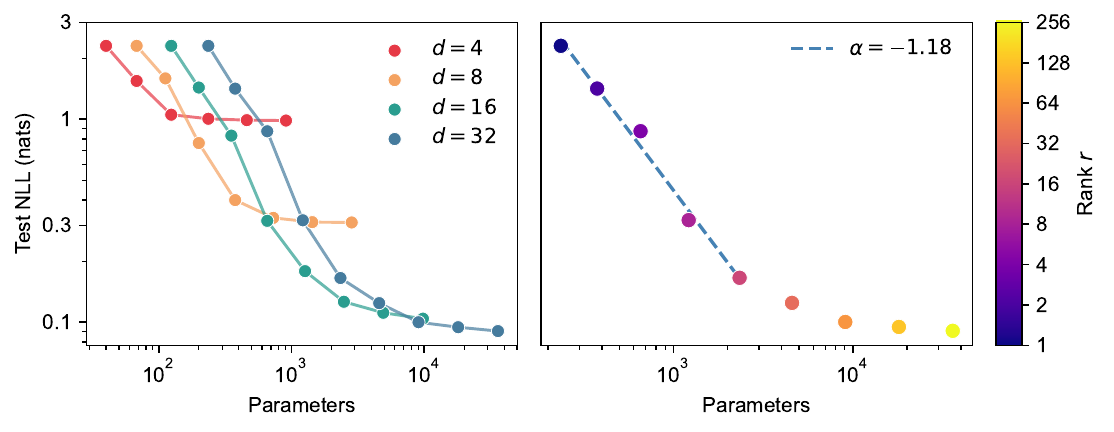}
    \caption{\textbf{Empirical scaling behavior of SNA classification on MNIST from pretrained latent spaces.} Total trainable parameter versus test negative log-likelihood (NLL) (\emph{left}) under fixed quadratic B-spline resolution $C=4$, across expanding dimensions $d$. Isolated power-law scaling fit (\emph{right}) for the $d=32$ manifold with slope $\alpha=-1.18$.}
    \label{fig:mnist}
\end{figure}

\section{Applications in engineering science}
\label{sec:applications}

Having been validated in Section \ref{sec:results}, the present section deploys the SNA and VSNA on two industrial engineering problems where the physical world-modeling capability is a key asset. Namely, the ability of this model to encode the complete solution manifold over space, time and parameters in a single trained object, then query, invert and propagate uncertainty through it -- never re-solving.

Section \ref{sec:lpbf} applies the VSNA to steady-state thermal field prediction in shaped-beam laser powder bed fusion of Stainless Steel 304L, treated here as a seven-dimensional spatiotemporal--parametric problem. Once trained, the model supports probabilistic geometry extraction and inverse process window identification under uncertainty -- tasks that would require millions of FEM re-solves by any classical approach. 

Section \ref{sec:ded} applies the CP-class SNA to process--structure forward and inverse modeling in directed energy deposition (DED) of Inconel 718. Mechanical properties are recovered from thermal histories with five orders of magnitude fewer parameters than prior approaches employing convolutional neural networks. The SNA demonstrates its unique ability to inversely generate plausible thermal histories from target properties in under 100ms on a commodity laptop CPU.

\begin{figure}
    \centering
    \includegraphics[width=\linewidth]{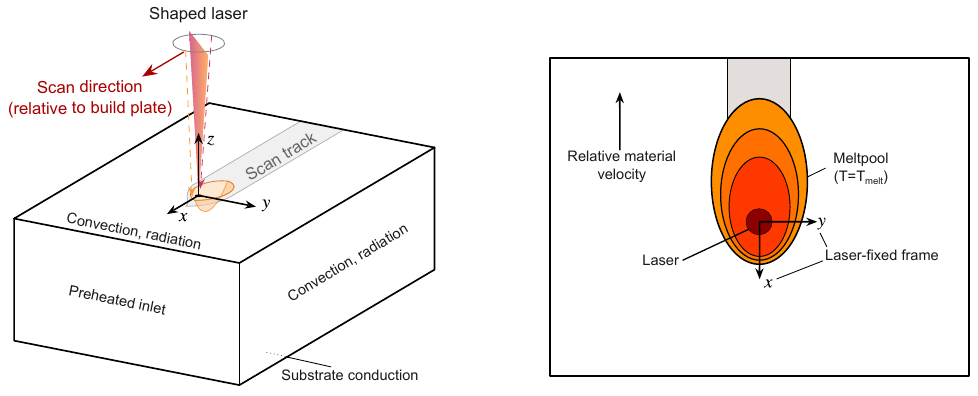}
    \caption{\textbf{Schematic and reference frame for the laser powder bed fusion (LPBF) model}. (\emph{left}) 3D spatial domain and boundary conditions including a preheated inlet, substrate conduction and convection and radiation under a shaped laser beam. (\emph{right}) Top-down view of the laser-fixed Eulerian frame of reference; the meltpool defined by the $T=T_\mathrm{melt}$ isotherm.}
    \label{fig:placeholder}
\end{figure}

\subsection{Designing beam shape in laser powder bed fusion for optimum melt pool under uncertainty}
\label{sec:lpbf}

Engineering the beam shape to obtain optimum microstructure and structure properties are becoming a new frontier for research in fusion-based manufacturing science \cite{abdelmoula2026laser,aydin2025situ,bi2023beam}. Commercially available Laser powder bed fusion (LPBF) machines, such as Laser Tech 30 by DMG MORI \cite{dmgmori_slm_powderbed}, has the capability of printing parts with hybrid beam shape. Although this area is an active field of research, preliminary results show that such beam engineering can produce interesting microstructure. However, this only adds to the already existing challenges in additive manufacturing namely design of process under uncertainty \cite{shang2025accurate,fang2025uncertainty}. Part qualification and control of additive manufacturing is a very high-dimensional problem as there are numerous variables, such as, feedstock material, laser power, scan speed, geometry, design, supply chain logistics, and so on, that can alter the part quality. Solving such a high-dimensional problem, either to predict thermal conditions as the part is build (forward problem) or to design material or process parameters (inverse problem), is a challenge for traditional methods without resorting to data-driven methods. This section presents VSNA as a vehicle to tackle such high-dimensional problem, in particular, how to design the beam shape to control or obtain specific meltpool geometry from first principles.  

The VSNA is applied to model the steady-state thermal fields induced by a moving heat source, focusing specifically on shaped-beam LPBF. In the laser-fixed frame of reference, the temperature deviation $\theta=T-T_\infty$ satisfies the advection--diffusion equation
\begin{align}
    -\rho c_p v\frac{\partial\theta}{\partial x} - \nabla\cdot(k\nabla\theta) = Q_\mathrm{beam}(x,y,z;P,s) \qquad \text{in } \Omega.
\end{align}
The volumetric heat source $Q_\mathrm{beam}$ deposits laser energy through the powder bed as a power-weighted superposition of a Gaussian core and annular ring beam,
\begin{align}
    Q_\mathrm{beam} = \eta P\left[(1-s)G(x,y,z) + sR(x,y,z)\right].
\end{align}
Here $\eta$ is absorptivity and $s\in[0,1]$ is the beam shape parameter: $s=0$ recovers a Gaussian spot, $s=1$ is a pure ring beam. The Gaussian source is separable $G(x,y,z) = g_x(x)g_y(y)g_z(z)$. $g_x$ and $g_y$ are Gaussian profiles of width $\sigma_x$ and $\sigma_y$ centered at the beam position $(x_b,y_b)$. $g_z$ is a depth profile decaying from the surface $z=z_1$ with penetration depth $\sigma_z$. The ring source $R_\mathrm{exact}=r(x,y)g_z(z)$ has radial profile
\begin{align}
    r(x,y)\propto\exp\!\left[-\frac{(\sqrt{(x-x_b)^2+(y-y_b)^2}-a)^2}{2\sigma_r^2}\right],
\end{align}
normalized so that $\int_\Omega R\,d\Omega=1$. This ring is non-separable in $x$ and $y$, and approximated by a rank-$m$ singular value decomposition
\begin{align}
    r(x,y)\approx\sum_{j=1}^m r_j^x(x)r_j^y(y).
\end{align}
The SNA framework naturally accommodates this as an $m$-term CP contribution. On all exposed surfaces, Robin conditions with linearized radiation model heat loss
\begin{align}
    -k\nabla\theta\cdot\mathbf{n} = h_\mathrm{eff}\theta, \qquad h_\mathrm{eff}=h+4\epsilon\sigma_\mathrm{SB}T_\infty^3
\end{align}
with separate coefficients $h_\mathrm{top,eff}$, $h_\mathrm{side,eff}$ and $h_\mathrm{bottom}$ on the top, side and substrate faces respectively. The inlet face imposes a preheating condition $-k\nabla\theta\cdot\mathbf{n}=h_\mathrm{inlet}(\theta-\theta_\mathrm{preheat})$.

\begin{table}[t]
\centering
\begin{tabular}{r|cccc}
\toprule
Model & Scope & Footprint & Walltime & VSNA speedup \\
\midrule
FEM  & Full 7D sweep &
$\approx2\times10^{11}$ DoFs &
$\approx 13$ days &
$1\times$ \\

VSNA & Full 7D manifold &
$20{,}448$ parameters &
$6$ s &
$\approx 190{,}000$ \\

FEM  & Single 3D solve &
$1{,}534{,}825$ DoFs &
$9$ s &
$1\times$ \\

VSNA & Single 3D query &
$20{,}448$ parameters &
$60~\mu$s &
$\approx 150{,}000$ \\
\bottomrule
\end{tabular}
\caption{\textbf{Computational comparison between sparse FEM and the VSNA manifold representation.} The FEM full-sweep estimate corresponds to serial evaluation over the complete $16\times16\times16\times32 = 131{,}072$ point parameter space. VSNA timings correspond to offline manifold training and online batched inference on a $180\times96\times64$ spatial grid. All results computed on a single NVIDIA A100 GPU (40GB).}
\label{tab:lpbf_comparison}
\end{table}

Four parameters are treated as free: scan speed $v$, thermal conductivity $k$, laser power $P$ and beam-shape parameter $s\in[0,1]$. This yields a seven-dimensional spatiotemporal--parametric domain $\mathcal{X}=\Omega\times\mathcal{P}$.

\begin{table}[htbp]
\centering
\small
\begin{tabular}{llll}
\toprule
\textbf{Category / Parameter} & \textbf{Symbol} & \textbf{Value / Range} & \textbf{Units} \\
\midrule
\emph{Material Properties (Stainless Steel 304L)} & & & \\
Density & $\rho$ & $7900$ & $\mathrm{kg/m^3}$ \\
Specific Heat Capacity & $c_p$ & $570$ & $\mathrm{J/(kg\cdot K)}$ \\
Volumetric Heat Capacity & $\rho c_p$ & $4.5\times 10^6$ & $\mathrm{J/(m^3\cdot K)}$ \\
Base Absorptivity & $\eta_0$ & $0.40$ & $-$ \\
Solidus / Melt Temperature & $T_s$ & $1673$ & $\mathrm{K}$ \\
Ambient / Preheating Temperature & $T_\infty$ & $300$ & $\mathrm{K}$ \\
\midrule
\emph{Spatial Domain Details ($\Omega$)} & & & \\
Domain Extents & $(x,y,z)$ & $[-1.5,0.5]\times[-0.5,0.5]\times[0,0.5]$ & $\mathrm{mm}$ \\
Domain Dimensions ($x \times y \times z$) & $-$ & $2.0 \times 1.0 \times 0.5$ & $\mathrm{mm}$ \\
Heat Transfer Coefficient (Top) & $h_{\mathrm{top}}$ & $25$ & $\mathrm{W/(m^2\cdot K)}$ \\
Heat Transfer Coefficient (Sides) & $h_{\mathrm{side}}$ & $25$ & $\mathrm{W/(m^2\cdot K)}$ \\
Linearized Radiation Coefficient & $h_{\mathrm{rad}}$ & $6.12$ & $\mathrm{W/(m^2\cdot K)}$ \\
Effective Top/Side Coefficient & $h_{\mathrm{eff}}$ & $31.12$ & $\mathrm{W/(m^2\cdot K)}$ \\
Substrate Boundary Coefficient & $h_{\mathrm{bottom}}$ & $1000$ & $\mathrm{W/(m^2\cdot K)}$ \\
\midrule
\emph{Parametric Manifold Subspace ($\mathcal{P}$)} & & & \\
Laser Power Range & $P$ & $[100,300]$ (Nominal: $240$) & $\mathrm{W}$ \\
Scan Velocity Range & $v$ & $[0.5,2.0]$ (Nominal: $0.8$) & $\mathrm{m/s}$ \\
Thermal Conductivity Range & $k$ & $[20,35]$ (Nominal: $30$) & $\mathrm{W/(m\cdot K)}$ \\
Beam Shape Parameter Range & $s$ & $[0.0,1.0]$ (Nominal: $0.0$) & $-$ \\
Gaussian Core Characteristic Radius & $\sigma_x,\sigma_y$ & $50$ & $\mu\mathrm{m}$ \\
Annular Ring Core Mean Radius & $a$ & $80$ & $\mu\mathrm{m}$ \\
Annular Ring Profile Width & $\sigma_r$ & $15$ & $\mu\mathrm{m}$ \\
Volumetric Penetration Depth & $\sigma_z$ & $50$ & $\mu\mathrm{m}$ \\
\bottomrule
\end{tabular}
\caption{Thermophysical properties, bounding spatial dimensions and parametric manifold search space configurations.}
\label{tab:lpbf_parameters}
\end{table}

\begin{figure}
    \centering
    \includegraphics[width=\linewidth]{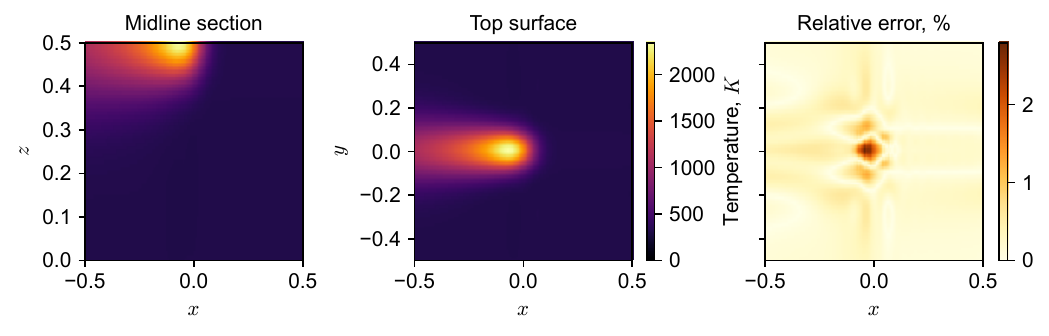}
    \caption{\textbf{VSNA validation against FEM at nominal parameters.} Midline cross-section (\emph{left}) and top surface (\emph{center}) of the learned temperature field $T$ queried at $v=0.8m/s,k=30\frac{W}{mK}, P=240W$ and $s=0$, a pure Gaussian beam. Pointwise relative errors against a finite-element reference (\emph{right}).}
    \label{fig:lpbf_validation}
\end{figure}

\begin{figure}
    \centering
    \includegraphics[width=\linewidth]{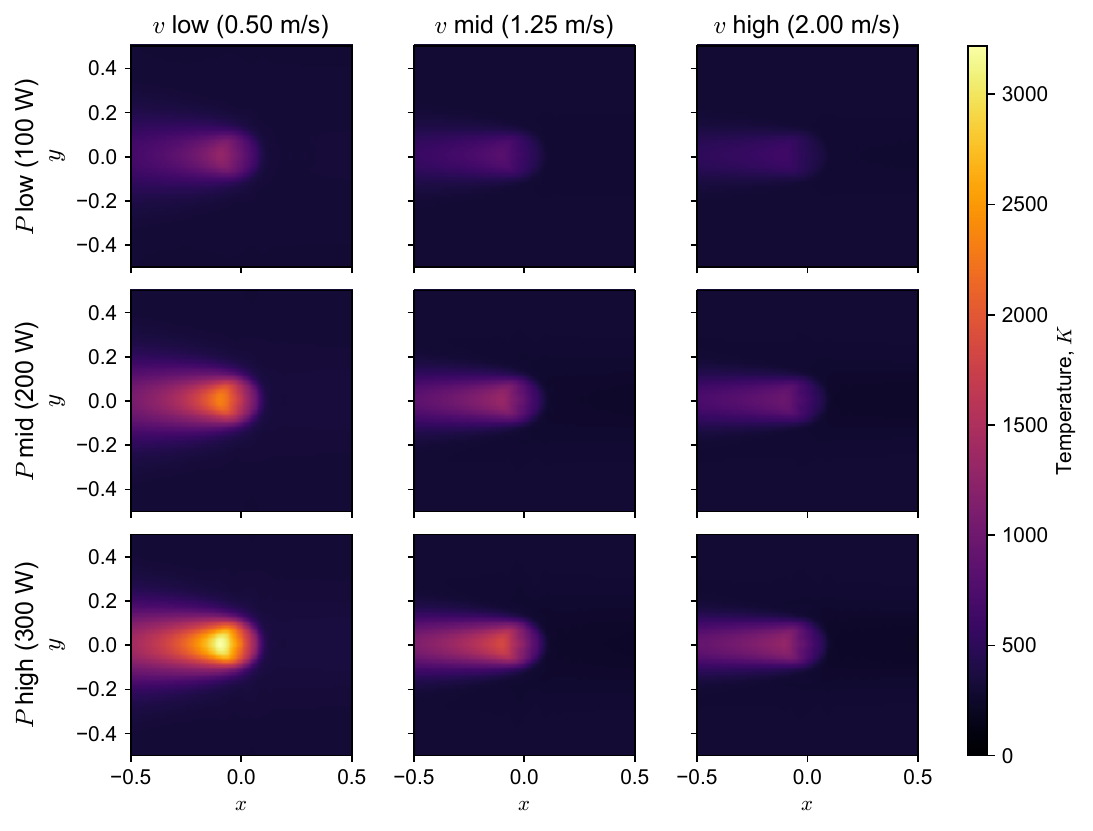}
    \caption{\textbf{Recovered thermal manifold across process parameter space.} Top-surface temperature rise $\theta$ queried from the pretrained VSNA across a $3\times3$ grid of low, medium and high scan speeds $v\in\{0.50,1.25,2.00\}m/s$ and tool power $P\in\{100,200,300\}W$ at fixed $k=27.5\frac{W}{mK}$ and $s=0.5$.}
    \label{fig:lpbf_manifold}
\end{figure}

The VSNA trial space is a CP-class separable representation over all seven coordinates, trained by minimizing the weak residual of the governing operator. The problem is linear with bilinear form with Robin conditions, which is bounded and coercive for $v,k>0$. Thus, all four variational guarantees apply. A reference solution is obtained by the finite-element method at fixed nominal parameters $(v_0,k_0,P_0, s_0)$ for validation. The VSNA, trained once, recovers the complete solution manifold and provides the temperature field at any queried $(v,k,P,s)$.

Once trained, the model supports direct probabilistic evaluation on the learned manifold without the need for resolves. Gaussian uncertainties are introduced in laser power $P=240(1+\epsilon_P),\epsilon_P\sim\mathcal{N}(0,0.03^2)$ and absorptivity $\eta=0.4(1+\epsilon_\eta),\epsilon_\eta\sim\mathcal{N}(0,0.04^2)$. The beam shape parameter is fixed at six discrete settings $s\in\{0,0.2,0.4,0.6,0.8,1\}$ and $k=30\frac{W}{mK}$. For each $s$, 1,000 Monte Carlo samples are drawn from $(\epsilon_P,\epsilon_\eta)$ and queried directly from the pretrained seven-dimensional manifold. Liquidus isotherms -- $T=T_\mathrm{melt}$ -- are extracted from both the top-surface and midline cross-section temperature fields to obtain stochastic meltpool geometries.

\begin{figure}
    \centering
    \includegraphics[width=\linewidth]{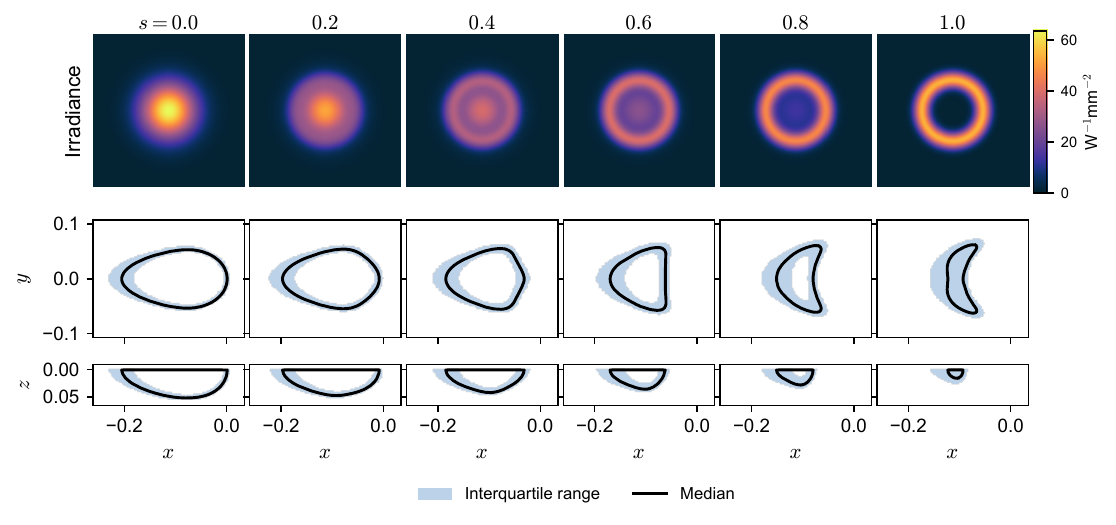}
    \caption{\textbf{Probabilistic meltpool geometries under process uncertainty.} Laser irradiance for the six discrete beam-shape settings $s\in\{0,0.2,0.4,0.6,0.8,1\}$ (\emph{top}). Top-surface (\emph{middle}) and midline cross-section (\emph{bottom}) liquidus isotherms extracted from Monte Carlo queries of the pretrained VSNA for the six beam-shape parameters. The black contour denotes the median boundary; the shaded region the interquartile range.}
    \label{fig:uncertainty}
\end{figure}

Figure \ref{fig:uncertainty} visualizes the resulting occupancy envelopes. The black contour denotes the median molten probability boundary, and shaded regions denote the interquartile range of molten geometries. As $s$ transitions from 0 to 1 -- the beam morphing from a Gaussian spot into a ring -- deep teardrop-shaped meltpools become progressively shorter, shallower and increasingly nonconvex. Simultaneously, the stochastic melt envelopes contract and the annular topology collapses. By $s=1$, the domain becomes simply connected as the central void closes. Lack-of-fusion thus enters the interquartile statistics, with 25\% of samples failing to exceed the liquidus temperature.

\begin{figure}
    \centering
    \includegraphics[width=\linewidth]{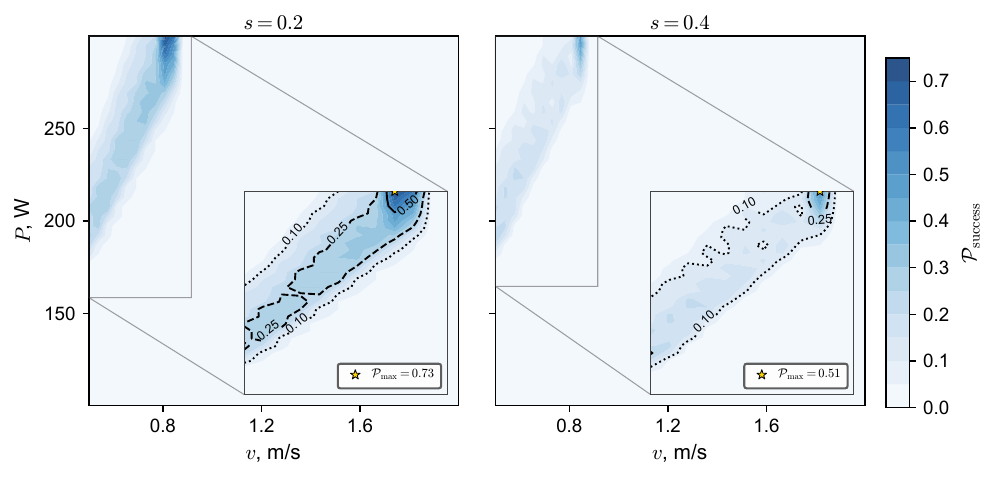}
    \caption{\textbf{Inverse process maps under uncertainty.} Contour maps of the probability ($\mathcal{P}_\mathrm{success}$) of producing a meltpool with width $130\pm10\mu m$ and depth $55\pm 5\mu m$ for combinations of $P\in[100,300]W$ and $v\in[0.5,2.0]m/s$. Only for two beam configurations $s=0.2$ (\emph{left}) and $s=0.4$ (\emph{right}) is this feasible. Contours in the expanded insets chart the reliability transitions $\mathcal{P}_\mathrm{success}\in\{0.10,0.25,0.50\}$. Gold stars identify the probability maximizers.}
    \label{fig:lpbf_inverse}
\end{figure}

The computational efficiency and continuous parameterization of the trained model enable the direct solution of inverse process-design problems conventionally intractable under uncertainty. Here, the combination of scan speed $v$ and laser power $P$ that maximizes likelihood of compliance $\mathcal{P}_\mathrm{success}$ with a target meltpool geometry is desired. The uncertainty is concentrated in absorptivity $\eta$, which is modeled as before $\eta=0.4(1+\epsilon_\eta),\epsilon_\eta\sim\mathcal{N}(0,0.04^2)$. 

The target criteria constraints both meltpool width $W^*=130\pm 10\mu m$ and depth $D^*=55\pm 5\mu m$. A dense parameter grid search is performed over $40\times40$ combinations of $(P,v)$. For each point on this grid, a $\eta$ is sampled at $N_{mc}=100$ points via Monte Carlo. Evaluating the entire grid requires 960,000 individual three-dimensional thermal solutions per beam shape configuration. These are discrete as before, namely $s\in\{0,0.2,0.4,0.6,0.8,1\}$. By exploiting the tensor-native structure of the VSNA surrogate, the entire sweep is vectorized at each query point. The feasibility probability is computed pointwise across the grid as the expected fraction of samples satisfying both geometric bounds concurrently,
\begin{equation}
    \mathcal{P}_{\text{success}}(v, P; s) = \frac{1}{N_{mc}} \sum_{n=1}^{N_{mc}} \mathbb{I}(|W_n - W^*| \leq \Delta W \cap |D_n - D^*| \leq \Delta D).
\end{equation}
$\mathbb{I}$ is the indicator function.

Figure \ref{fig:lpbf_inverse} illustrates the resulting probability maps. Out of all evaluated configurations, viable process windows emerge only for the intermediate shapes $s=0.2$ and $s=0.4$. The standard configurations, Gaussian spot and annular ring, fail entirely; there does not exist a feasible region. The grid search identifies $s=0.2$ as the most robust operational envelope. The global reliability optimum here is $\mathcal{P}_\mathrm{max}=0.730$ at $v^*=0.808m/s$ and $P^*=300W$ yielding a meltpool structure with expected width $W=128.3\mu m$ and depth $D=54.0\mu m$. Increasing the ring weight to 0.4 shifts the scan speed to $v^*=0.846m/s$ with a lower success probability of $\mathcal{P}_\mathrm{max}=0.510$.

It is notable that the entire sweep of almost 1,000,000 distinct three-dimensional evaluations across the parameter space completed in 102s, solely on commodity laptop hardware (Apple M1 CPU, 8GB RAM). A classical approach would take over three months, requiring a million full re-solves each 9s on an A100 GPU.

\subsection{Prediction and inversion in directed energy deposition}
\label{sec:ded}

Following the data-free forward and inverse design example of the previous section, CP-class SNA is next demonstrated on a data-driven process--structure inverse modeling problem investigated originally in \cite{xie2021mechanistic, fang2022data}, linking thermal histories recorded during directed energy deposition (DED) of Inconel 718 thin-walls to spatially varying mechanical properties of the resultant print. Inverse modeling of 3D printed structures has been a challenge for engineers, as seemingly minor perturbations in the process may result in drastically varying mechanical properties within a single part. A schematic of the experimental setup for this problem is presented in Fig. \ref{fig:schematic}. The raw thermal signals, recorded by infrared (IR) sensors, are stochastic and nonlinear, as well as long ($10,000$ time indices), whilst available paired data (yield strength, ultimate tensile strength (UTS), and elastic modulus) are sparse (96 samples). The mechanical properties were determined via uniaxial testing of small tensile coupons at different points of a thin wall built with DED. In this example, we attempt to map back to the thermal history at a specific material point with quantified uncertainty for a target mechanical property. 

\begin{figure}[!t]
    \centering
    \includegraphics[width=0.55\linewidth]{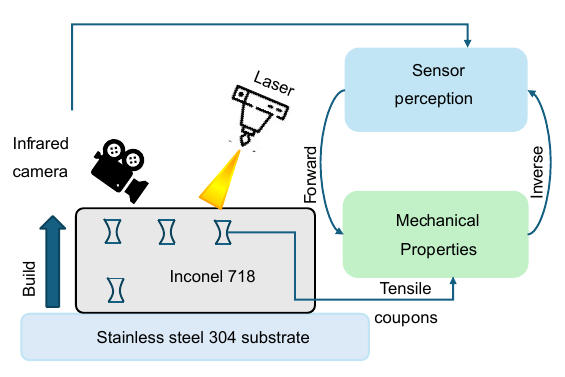}
    \caption{\textbf{Experimental schematic}. A laser directed energy deposition machine builds thin-walled structures layer by layer on a stainless steel 304 substrate, whilst an infrared camera records the evolving thermal field during the build. These measurements are subsequently linked to the mechanical response of the material through tensile testing of extracted coupons.}
    \label{fig:schematic}
\end{figure}

\begin{figure}[!t]
    \centering
    \includegraphics[width=0.55\linewidth]{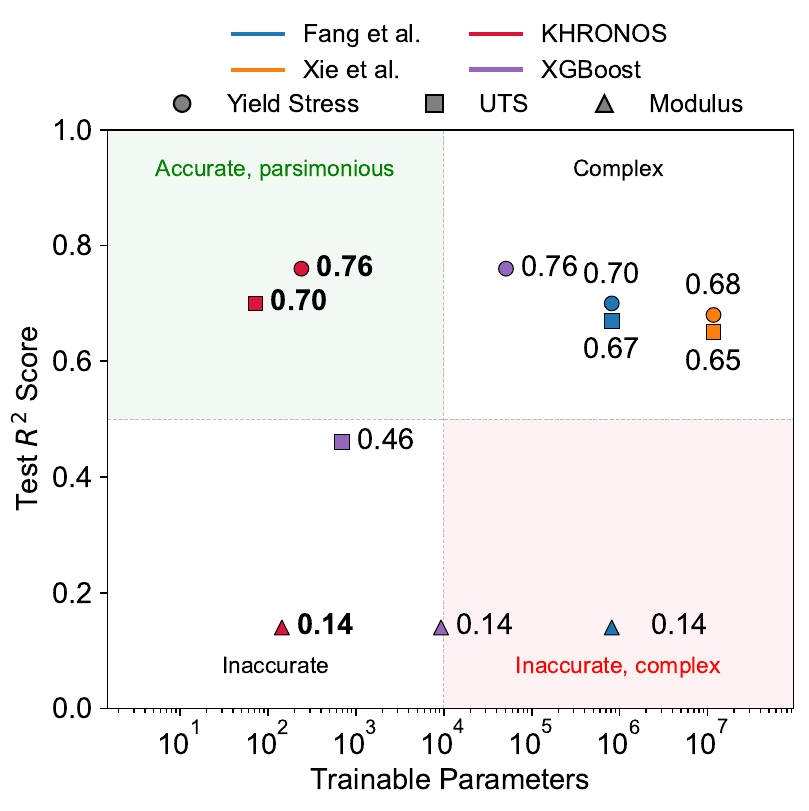}
    \caption{\textbf{Comparative performance.} Predictive performance versus trainable parameters on the Inconel 718 thermal-history dataset. The SNA matches or exceeds prior model accuracy on both yield stress (YS) and ultimate tensile strength (UTS) with up to five orders-of-magnitude fewer parameters \cite{xie2021mechanistic, fang2022data}, and three orders fewer than XGBoost \cite{chen2016xgboost}.}
    \label{fig:comparison}
\end{figure}

\begin{figure}[!t]
    \centering
    \includegraphics[width=\linewidth]{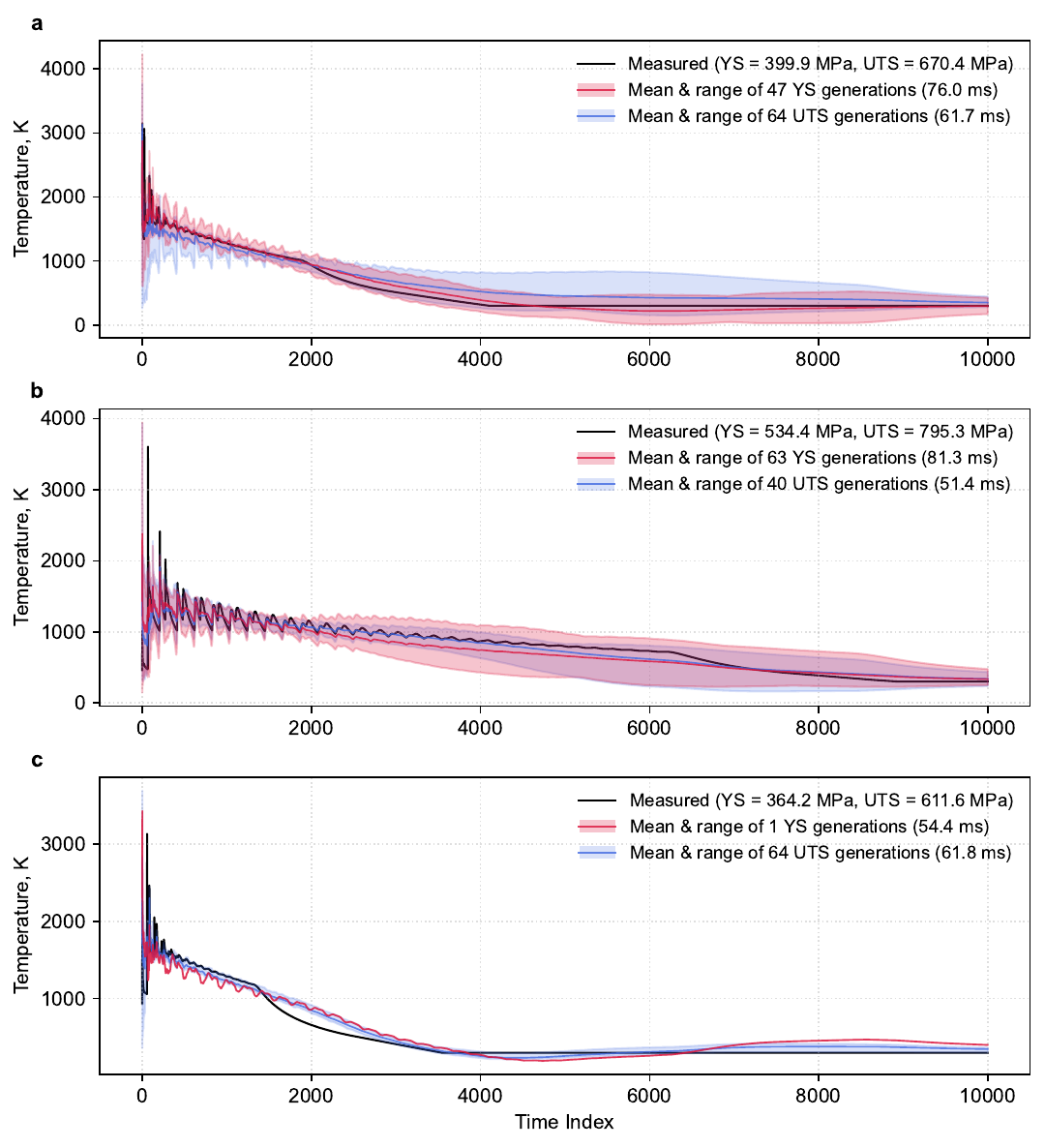}
    \caption{\textbf{Generated thermal histories.} Generative inversion of three target mechanical properties -- yield stress and ultimate tensile stress -- to thermal histories, with mean and range illustrated. This signal can be used as the target during in situ monitoring and control of the DED process.}
    \label{fig:inversion}
\end{figure}

Thermal data preprocessing follows \cite{xie2021mechanistic} beginning with a wavelet transform. Previous works fed these high-dimensional representations into monolithic convolutional neural networks (CNNs) -- with approximately 11 million parameters in the modified ResNet18 of \cite{xie2021mechanistic} and 800,000 in the one-dimensional CNN of \cite{fang2022data}. A subsequent principal component analysis (PCA) is introduced, decomposing the preprocessed representations into a low-dimensional latent space. This coordinate transform reveals the factorizability of the thermal physics. Exploiting this, a CP-class SNA required only 240 parameters for yield stress (YS) and 108 for ultimate tensile strength (UTS) -- a reduction of four to five orders of magnitude. Despite this parsimony, the models achieved test $R^2$ scores of 0.76 (YS) and 0.70 (UTS), matching or exceeding prior approaches; the only model to achieve the highest score on both metrics. A comparative summary, including an XGBoost baseline \cite{chen2016xgboost}, is visualized in Fig. \ref{fig:comparison}. All models exhibit poor performance for the material modulus, saturating at $R^2=0.14$. This is consistent with prior studies: elastic modulus is largely composition-controlled and therefore poorly resolved by thermal histories alone, which are further compounded by IR sensor noise and uncertainties inherent to the deposition process.

Inverting opaque monolithic models requires expensive surrogate optimization or the training of separate inverse networks. The SNAs, however, admits easily traced (even analytic) derivatives and is highly lightweight. These properties enable target mechanical properties to be inverted to plausible thermal histories via a structured Newton search. Multiple initializations produce a low-dimensional manifold of solutions consistent with the queried property. The resulting inversions recover ensembles of thermal histories that closely resemble the ground-truth trajectory with a reasonable uncertainty envelope. This capability is illustrated in Fig. \ref{fig:inversion} across three target pairs of target yield- and ultimate tensile stresses. The search is so lightweight, in fact, that 64 candidate histories are generated in under 100ms on commodity CPU hardware, albeit not all converging. This suggests a route towards an intelligent feedback control system that, deployed in a 3D printing machine, could maintain target mechanical properties throughout the build irrespective of the printing strategy.

\section{Discussion and future horizons}

The theoretical grounding and empirical validation of the SNA in high-dimensional embedding spaces -- its variational counterpart across elliptic, parabolic and hyperbolic PDE manifolds -- demonstrates its versatility as a physical world model. Beyond this applications, the SNA serves as a foundational building block for more complex, composite system architectures. Their structural flexibility has already begun to manifest across three distinct frontiers of deep learning: downstream generative inverse design, sequence modeling and reinforcement learning.

\paragraph{Bidirectional generative--predictive integration}

In the domain of physical engineering and inverse design, a key challenge is in the bidirectional mapping between high-dimensional process design spaces and desired properties. A single-dimensional example is presented in Section \ref{sec:ded}, but extension to two-dimensional pixel and three-dimensional voxel requires corresponding spatial inductive biases that the SNA, alone, lacks. The \emph{Janus} framework \cite{batley2026janus} directly leverages the SNA as predictive head embedded within a latent autoencoding generative loop. Inversion of the SNA produces a latent vector given according target properties. Generated vectors are then decoded to produce pixel outputs. This underscores the role of the SNA as an optimization-fluent physical world model capable of guiding generative physics engines. 

\paragraph{Smooth token embedding in sequence models}

Further evidence of the SNA's domain-agnostic capability is in applicability in discrete token sequence modeling. The \emph{Leviathan} Transformer architecture \cite{batley2026leviathan} extracts the core inductive bias of coordinate-wise separability where traditional Transformers suffer from severe parameter bloat and inflexibility due to tied-embedding lookup tables. Vocabulary manifolds benefit from the same low-interaction and smoothness properties shown here to be well-suited to solving PDEs, dramatically accelerating sample efficiency across the long tail of rare data.

\paragraph{Smooth actor-critics in reinforcement learning}

The coordinate-wise decoupling of the SNA provides an immediate structural advantage in continuous control. Spline-based adaptive networks (SPAN) \cite{mostakim2026agile} embeds the architecture directly within actor-critic loops to parameterize the policy and value functions. Standard deep reinforcement learning architectures often suffer from sample inefficiency and optimization instability due to the non-convexity of monolithic neural network policies. The inductive bias of the SNA when combined with a dense preprocessing layer is shown to accelerate policy convergence and parameter-efficiency.

When read collectively, these extensions shine light upon a deeper theoretical symmetry. The mathematical machinery shown to ameliorate the curse of dimensionality in high-dimensional PDEs is similar to that shown to navigate the latent bottlenecks of computer vision, the rare long-tailed distributions of natural language and the non-convex landscapes of reinforcement learning. In each domain, dense neural architectures over-allocate parameter capacity to high-order coordinate interactions that carry negligible physical or statistical significance. By explicitly governing these interactions via $\mathcal{C}$, the SNA better isolates the active low-dimensional manifolds undergirding these systems. Consequently, this architecture transitions deep learning towards a mathematically tractable, sample-efficient and structurally verifiable framework for physical world modeling.

\section{Conclusions}

This work introduced the Separable Neural Architecture (SNA) and its variational extension (VSNA) to bridge the long-standing gap between discrete tensor decompositions and neural approximation theory. The SNA decouples low-dimensional functional complexity -- assigned to localized coordinate atoms $\phi^{(S)}$ -- from global interactivity -- governed by a sparse, low-rank interaction object $\mathcal{C}$ -- to restrict parameter allocation to active coordinate interactions. This structured formulation is a unifying formalism, encompassing, amongst others, additive models, tensor-decomposed models and variational DeepONet. In certain cases, the SNA provides a universal approximation guarantee established via the Stone-Weierstrass theorem.

When deployed as a Galerkin trial space over spatiotemporal--parametric domains $\mathcal{X}$, the resulting VSNA problem provides the standard variational guarantees under the Lax--Milgram theorem. By exploiting a tensor-native alternating least-squares (ALS) optimization algorithm, the computational complexity of solving high-dimensional PDEs is reduced from the exponential cost of grid-based methods to polynomial.

Empirical validation across the canonical triad of PDE characters -- elliptic Poisson convergence, hyperbolic inviscid- and viscous Burgers shock-layer formations and a parabolic six-dimensional parametric advection-diffusion system -- demonstrated agreement with predicted algebraic and spectral scaling laws. The VSNA solved the six-dimensional advection-diffusion with high-fidelity continuous coverage, whilst compressing parameter counts a million-fold compared to linear FEM. When deployed as a data-driven model in regression problems, the SNA efficiently isolates the underlying physics where black-box architectures struggle.

In seven-dimensional laser powder-bed fusion (LPBF) modeling, the VSNA captured the complete steady-state thermal manifold across non-separable adaptive-beam geometries. This enabled a million-query Monte Carlo inverse optimization sweep to run in just 102 seconds on commodity laptop CPU hardware -- a $150{,}000\times$ speedup over traditional finite element multi-query evaluations run on an NVIDIA A100 GPU. The SNA was deployed on a sparse industrial dataset, where it mapped directed energy deposition thermal histories to Inconel 718 tensile properties with five orders of magnitude fewer parameters than the convolutional architectures employed in prior work. Because it admits tractable derivatives, the SNA was shown to be able to inversely generate thermal histories given desired properties in under 100ms on commodity laptop CPU hardware.

The Separable Neural Architecture is thus demonstrated to satisfy the criteria of a physical world model: a single, compact mathematical representation that maps over entire solution manifolds to enable real-time inversion, optimization loops and rapid uncertainty quantification.

\section{Acknowledgments}
S.S. and R.B would like to acknowledge Virginia Tech Start-up Fund, Virginia Tech Foundation Fund, and Prof Ella Atkins Faculty Fellowship for supporting the work. The authors thank Virginia Tech Advanced Research Computing (ARC) for supporting the work by providing computational resources.
\clearpage

\appendix

\section{Universality of tensor decomposed Separable Neural Architectures}\label{App:A}

\begin{theorem}[Universal approximation of CP-class Separable Neural Architectures]
\label{thm:cpuniv}
Let $\Omega=[0,1]^d$. For each coordinate $i\in\{1,\dots,d\}$, let $U^{(i)}\subset C([0,1])$
be a unital subalgebra that is dense in $C([0,1])$ with respect to $\|\cdot\|_\infty$.
Define
\begin{align}
    \mathcal A = \left\{\sum_{j=1}^r \prod_{i=1}^d a^{(i)}_j(x_i): a^{(i)}_j\in U^{(i)}, r\in\mathbb N \right\}.
\end{align}
Then $\mathcal A$ is dense in $C([0,1]^d)$ with respect to the uniform norm.
\end{theorem}

\begin{proof}
For each coordinate $i\in\{1,\dots,d\}$, let $\tilde U^{(i)}=\operatorname{span}\{\phi^{(i)}_\alpha : \alpha=1,\dots,r\}$ be a function space dense in $C([0,1])$, and let $U^{(i)}=\operatorname{alg}(\tilde U^{(i)})$ be the univariate algebra generated by $\tilde U^{(i)}$, with $1\in U^{(i)}$. Pick arbitrary $f,g\in\mathcal{A}$ such that
\begin{align}
    f(x)=\sum_{j=1}^{r}\ \prod_{i=1}^{d} \alpha^{(i)}_{j}(x_i),~~ g(x)=\sum_{k=1}^{s}\ \prod_{i=1}^{d} \beta^{(i)}_{k}(x_i),
\end{align}
with $\alpha^{(i)}_{j},\beta^{(i)}_{k}\in U^{(i)}$. Then,
\begin{align}
    f(x)+g(x)&=\sum_{j=1}^{r}\ \prod_{i=1}^{d} \alpha^{(i)}_{j}(x_i)+\sum_{k=1}^{s}\ \prod_{i=1}^{d} \beta^{(i)}_{k}(x_i),\\
    &=\sum_{l=1}^{r+s}\prod_{i=1}^d\gamma_l^{(i)}(x_i)\quad\text{where }~\gamma_l^{(i)}\in U^{(i)}.
\end{align}
Therefore $f+g\in\mathcal{A}$, confirming closure under addition. Take an arbitrary element $f\in\mathcal{A}$ as before, and multiply by an arbitrary scalar $\lambda\in\mathbb{R}$ so
\begin{align}
    \lambda f(x)=\sum_{j=1}^{r}\lambda \prod_{i=1}^{d} \alpha^{(i)}_{j}(x_i)=\sum_{j=1}^{r} \prod_{i=1}^{d} \tilde\alpha^{(i)}_{j}(x_i),
\end{align}
with $\tilde\alpha^{(1)}_{j}=\lambda\alpha^{(1)}_{j}$, all else unchanged. Since $U^{(i)}$ is a vector space, $\lambda\alpha^{(1)}_{j}\in U^{(i)}$, thus $\lambda f\in\mathcal{A}$ and $\mathcal{A}$ is closed under scalar multiplication. Pick $f,g\in\mathcal{A}$ as before. Then, their product is
\begin{align}
    (fg)(x)&=\left(\sum_{j=1}^{r}\ \prod_{i=1}^{d} \alpha^{(i)}_{j}(x_i)\right)\left(\sum_{k=1}^{s}\ \prod_{i=1}^{d} \beta^{(i)}_{k}(x_i)\right),\\
    &=\sum_{j=1}^{r}\sum_{k=1}^{s}\prod_{i=1}^{d}\gamma_{jk}^{(i)}(x_i),~ \gamma_{jk}^{(i)}=\alpha^{(i)}_{j}(x_i)\beta^{(i)}_{k}(x_i).
\end{align}
Recall $U^{(i)}=\operatorname{alg}(\tilde U^{(i)})$ is a unital algebra and so it follows by definition that $\gamma_{jk}^{(i)}\in U^{(i)}$. Therefore, $(fg)(x)\in\mathcal{A}$, and $\mathcal{A}$ is an algebra: closed under multiplication. Since $1\in U^{(i)}$, it follows that $\mathcal{A}$ is a unital algebra.

Take two distinct points $x,y\in[0,1]^d$, $x=(x_1,\dots,x_d),~y=(y_1,\dots,y_d), $ with $x_i\neq y_i$ for at least one $i=1,\dots,d$. Now, because $U^{(i)}$ is dense in $C([0,1])$ and $C([0,1])$ separates points, $U^{(i)}$ also separates points: $\exists~u\in U^{(i)}$ with $u(x_i)\neq u(y_i)$. Then the separable function $f(z)=u(z_i)\in\mathcal{A}$ and satisfies $f(x)\neq f(y)$. Therefore, $\mathcal{A}$ separates points.

Any unital subalgebra on $C([0,1]^d)$ that separates points is itself dense in $C([0,1]^d)$, given that $[0,1]^d$ is a compact Hausdorff space. This follows from Stone-Weierstrass. Indeed, $\mathcal{A}$ is dense in $C([0,1]^d)$; for a given $f^*\in C([0,1]^d)$ and any $\varepsilon>0$, density yields $f\in\mathcal{A}$ such that
\begin{align}
    \|f^*-f\|_\infty<\varepsilon.
\end{align}
By definition of $\mathcal{A}$, $f$ is a \emph{finite} sum of products. Therefore, there exists $r$ and functions $\{\phi^{(i)}_{j}\}$ such that
\begin{align}
\sup_{x\in[0,1]^d}\left|f^*(x)-\sum_{j=1}^{r}\prod_{i=1}^{d}\phi^{(i)}_{j}(x_i)\right|<\varepsilon,
\end{align}
which is the desired universal approximation statement.
\end{proof}

\begin{corollary}[Universal approximation of tensor decomposed Separable Neural Architectures]
\label{cor:tduniv}
Let $\mathcal{F}_{TD}$ represent the functional class of tensor decomposed (TD)-class Separable Neural Architectures on $[0,1]^d$. Then $\mathcal{F}_{TD}$ is dense in $C([0,1]^d)$ with respect to the uniform norm $\|\cdot\|_\infty$.
\end{corollary}

\begin{proof}
By definition, the Canonical Polyadic (CP) class $\mathcal{A}$ is a specific, restricted instantiation of the broader TD-class $\mathcal{F}_{TD}$. The general multidimensional atoms $\phi^{(j)}(x)$ permitted in the TD-class are strictly constrained in the CP-class to be factorized products of univariate subatoms, $\prod_{i=1}^d\phi^{(i)}_{j}(x_i)$. 

Because any valid CP-class function is inherently a valid TD-class function, the functional classes satisfy the strict inclusion relation $\mathcal{A} \subseteq \mathcal{F}_{TD}$. 

Theorem \ref{thm:cpuniv} establishes that the subset $\mathcal{A}$ is dense in $C([0,1]^d)$ under the uniform norm. It is a fundamental property of dense sets that any superset containing a dense subset is itself dense in the same space. Therefore, the superset $\mathcal{F}_{TD}$ must also be dense in $C([0,1]^d)$. 

Consequently, for any target function $f^* \in C([0,1]^d)$ and any $\varepsilon > 0$, there exists an approximator $f \in \mathcal{F}_{TD}$ with a finite rank $r$ such that $\|f^* - f\|_\infty < \varepsilon$.
\end{proof}

\section{Theoretical guarantees of the Variational Separable Neural Architecture}

Let $V$ be a Hilbert space corresponding to the spatiotemporal--parametric domain $\mathcal{X}$, equipped with the norm $\|\cdot\|_V$. Consider a variational problem defined by a bilinear form $a: V \times V \to \mathbb{R}$ and a linear functional $\ell: V \to \mathbb{R}$. The exact weak formulation seeks $u \in V$ such that
\begin{equation}
    a(u,v) = \ell(v) \quad \forall v \in V.
\end{equation}

The standard assumptions of regularity are imposed on the governing operator:
\begin{enumerate}
    \item \textbf{Boundedness:} There exists a constant $c_1 > 0$ such that $|a(u,v)| \leq c_1 \|u\|_V \|v\|_V$ for all $u,v \in V$.
    \item \textbf{Coercivity:} There exists a constant $c_0 > 0$ such that $a(v,v) \geq c_0 \|v\|_V^2$ for all $v \in V$.
    \item \textbf{Bounded linear functional:} $\ell \in V^*$, meaning there exists a constant $M > 0$ such that $|\ell(v)| \leq M \|v\|_V$ for all $v \in V$.
\end{enumerate}

Under these conditions, the Lax-Milgram theorem guarantees a unique exact solution $u \in V$. The finite-dimensional Canonical Polyadic (CP) trial space $\mathcal{F}_r \subset V$ of rank $r$ is constructed from univariate subatoms $\psi_i^{(j)} \in V^{(i)}$, and the Galerkin approximation $u_r \in \mathcal{F}_r$ is sought satisfying:
\begin{equation}
    a(u_r, v_r) = \ell(v_r) \quad \forall v_r \in \mathcal{F}_r.
\end{equation}

\begin{theorem}[Well-posedness]
    The VSNA Galerkin approximation admits a unique solution $u_r \in \mathcal{F}_r$.
\end{theorem}
\begin{proof}
    For a fixed rank $r$ and fixed basis representations of the subatoms, the VSNA trial space $\mathcal{F}_r$ constitutes a finite-dimensional subspace of the Hilbert space $V$. Because $\mathcal{F}_r \subset V$, the bilinear form $a(\cdot,\cdot)$ remains bounded and coercive when restricted to $\mathcal{F}_r \times \mathcal{F}_r$. Similarly, the linear functional $\ell$ remains bounded on $\mathcal{F}_r$. By the application of the Lax-Milgram theorem to the finite-dimensional closed subspace $\mathcal{F}_r$, there exists a unique $u_r \in \mathcal{F}_r$ that satisfies the restricted variational problem.
\end{proof}

\begin{theorem}[Quasi-optimality]
\label{thm:qo}
    Let $u \in V$ be the exact solution and $u_r \in \mathcal{F}_r$ be the VSNA Galerkin solution. The approximation error is strictly bounded by the best possible approximation within the separable trial space:
    \begin{equation}
        \|u - u_r\|_V \leq \frac{c_1}{c_0} \min_{v_r \in \mathcal{F}_r} \|u - v_r\|_V.
    \end{equation}
\end{theorem}
\begin{proof}
    Because $\mathcal{F}_r \subset V$, $v = v_r \in \mathcal{F}_r$ may be chosen as a test function in the exact problem, yielding $a(u, v_r) = \ell(v_r)$. Subtracting the Galerkin formulation $a(u_r, v_r) = \ell(v_r)$ from this yields the fundamental Galerkin orthogonality condition:
    \begin{equation}
        a(u - u_r, v_r) = 0 \quad \forall v_r \in \mathcal{F}_r.
    \end{equation}
    For any function $v_r \in \mathcal{F}_r$, coercivity gives:
    \begin{equation}
        c_0 \|u - u_r\|_V^2 \leq a(u - u_r, u - u_r).
    \end{equation}
    Adding and subtracting $v_r$, and applying the bilinearity of $a(\cdot,\cdot)$ gives
    \begin{equation}
        a(u - u_r, u - u_r) = a(u - u_r, u - v_r) + a(u - u_r, v_r - u_r).
    \end{equation}
    Since $(v_r - u_r) \in \mathcal{F}_r$, the second term vanishes due to Galerkin orthogonality. Applying the boundedness of the bilinear form to the remaining term yields
    \begin{equation}
        c_0 \|u - u_r\|_V^2 \leq a(u - u_r, u - v_r) \leq c_1 \|u - u_r\|_V \|u - v_r\|_V.
    \end{equation}
    Dividing by $c_0 \|u - u_r\|_V$ gives $\|u - u_r\|_V \leq \frac{c_1}{c_0} \|u - v_r\|_V$. Because this holds for any $v_r \in \mathcal{F}_r$, the minimum over all $v_r \in \mathcal{F}_r$ may be taken, recovering Céa's Lemma for the Separable Neural Architecture.
\end{proof}

\begin{theorem}[Convergence]
    Assume $k=d$. If each univariate subatom family $\psi_i^{(j)}$ is dense in its respective coordinate space $V^{(i)}$, the VSNA Galerkin solution $u_r$ converges to the exact solution $u$ as the interaction rank $r \to \infty$. That is, $\lim_{r \to \infty} \|u - u_r\|_V = 0$.
\end{theorem}
\begin{proof}
    As established in Corollary \ref{cor:tduniv} and its extension to Hilbert tensor product $V=\bigotimes_{i=1}^dV^{(i)}$, the canonical polyadic tensor class $\bigcup_{r=1}^\infty \mathcal{F}_r$ is dense in the joint Hilbert space $V = \bigotimes_{i=1}^d V^{(i)}$ under the assumption that the univariate bases are dense in their respective 1D spaces. Therefore, for any $\varepsilon > 0$, there exists a finite rank $r^*$ and an approximator $v_{r^*} \in \mathcal{F}_{r^*}$ such that the best-approximation error satisfies:
    \begin{equation}
        \min_{v_{r^*} \in \mathcal{F}_{r^*}} \|u - v_{r^*}\|_V < \frac{c_0}{c_1} \varepsilon.
    \end{equation}
    Applying the quasi-optimality bound from Theorem \ref{thm:qo}, the error of the Galerkin solution $u_{r^*}$ is bounded by:
    \begin{equation}
        \|u - u_{r^*}\|_V \leq \frac{c_1}{c_0} \left( \frac{c_0}{c_1} \varepsilon \right) = \varepsilon.
    \end{equation}
    As rank $r \to \infty$, the best-approximation error monotonically approaches zero, guaranteeing that the Galerkin solution converges to the exact weak solution $u$.
\end{proof}

\begin{theorem}[Stability]
    The VSNA Galerkin solution $u_r$ satisfies the absolute stability bound:
    \begin{equation}
        \|u_r\|_V \leq \frac{1}{c_0} \|\ell\|_{V^*}.
    \end{equation}
\end{theorem}
\begin{proof}
    Choosing $v_r = u_r$ in the Galerkin formulation yields $a(u_r, u_r) = \ell(u_r)$. Applying the coercivity of $a(\cdot,\cdot)$ to the left-hand side and the definition of the dual norm $\|\ell\|_{V^*} = \sup_{v \neq 0} \frac{\ell(v)}{\|v\|_V}$ to the right-hand side gives
    \begin{equation}
        c_0 \|u_r\|_V^2 \leq a(u_r, u_r) = \ell(u_r) \leq \|\ell\|_{V^*} \|u_r\|_V.
    \end{equation}
    Dividing both sides by $c_0 \|u_r\|_V$ (assuming $u_r \neq 0$) yields the desired stability bound, demonstrating that the learned separable representation is continuously dependent on the problem data and cannot diverge unless the source functional itself is unbounded.
\end{proof}

\bibliographystyle{elsarticle-num-names} 
\bibliography{references.bib}

\end{document}